\documentclass{article}
\usepackage{arxiv}
\usepackage[utf8]{inputenc}
\usepackage[english]{babel}
\usepackage[ruled]{algorithm2e}
\usepackage{longtable}
\usepackage{placeins}
\usepackage{amsmath}
\usepackage{amsthm}
\usepackage{graphicx}
\usepackage{color}
\usepackage{geometry}
\usepackage{adjustbox}
\usepackage{mathtools}
\usepackage{longtable}
\usepackage{placeins}
\usepackage{adjustbox}
\usepackage{booktabs}
\usepackage{placeins}
\usepackage{lineno}
\usepackage{setspace} 
\usepackage{hyperref}

\usepackage{mathtools}
\usepackage{booktabs}
\usepackage[ruled]{algorithm2e}
\usepackage{amsthm}

\title{An Exact Hypergraph Matching Algorithm for Nuclear Identification in Embryonic \textit{Caenorhabditis elegans}}

\author{
 Andrew Lauziere \\
  Department of Mathematics\\
  University of Maryland, College Park\\
  College Park, MD 20742 \\
  \texttt{lauziere@umd.edu} \\
   \And
 Ryan Christensen \\
  Laboratory of High Resolution Optical Imaging\\
  National Institutes of Health\\
  Bethesda, MD 20892 \\
  \texttt{ryan.christensen@nih.gov} \\
  \And
 Hari Shroff \\
  Laboratory of High Resolution Optical Imaging\\
  National Institutes of Health\\
  Bethesda, MD 20892 \\
  \texttt{hari.shroff@nih.gov} \\
  \And
 Radu Balan \\
  Center for Scientific Computation and Mathematical Modeling \\
  University of Maryland, College Park\\
  College Park, MD 20742 \\
  \texttt{rvbalan@umd.edu} \\}

\date{}

\begin{document}

\maketitle

\begin{abstract}
Point-set matching is a common task in computer vision. Many matching applications feature affine point transformations which can be adequately modeled via lower order objective functions \cite{zhou_factorized_2016,zhang_kergm_2019}. However, point-set matching tasks may require richer detail in order to characterize underlying relationships. Hypergraphs, an extension of traditional graphs, have emerged to more intricately model relationships between points. Existing hypergraphical point-set matching methods are limited to heuristic algorithms which do not easily scale to handle higher degree hypergraphs \cite{duchenne_tensor-based_2010,chertok_efficient_2010,lee_hyper-graph_2011}. Our proposed algorithm, \textit{Exact Hypergraph Matching} (\textit{EHGM}), adapts the classical branch-and-bound paradigm to dynamically identify a globally optimal correspondence between point-sets under an arbitrarily intricate hypergraphical model. \textit{EHGM} is applicable to conservatively sized ($n \leq 20$) point-set matching problems in which relationships between points require increased context to adequately characterize. The methodology is motivated by \textit{Caenorhabditis elegans}, a model organism used frequently in developmental biology and neurobiology \cite{white_structure_1986, sulston_embryonic_1983, chisholm_genetics_2016, rapti_perspective_2020}. The \textit{C. elegans} embryo can be used for cell tracking studies to understand how cell movement drives the development of specific embryonic tissues. However, twitching due to muscular activity in late-stage embryos invalidates traditional cell tracking approaches. The embryo possesses a small set of cells which together act as fiducial markers to approximate the coiled embryo's posture, serving as a frame of reference to track cells of various tissues during late-stage embryogenesis \cite{christensen_untwisting_2015}. Current approaches to posture identification rely on time-consuming manual efforts by trained users which limits the efficiency of subsequent cell tracking. \textit{EHGM} with biologically inspired hypergraphical models identifies posture more accurately than established point-set matching methods, correctly identifying twice as many sampled postures as a heuristic graphical approach.  

\end{abstract}

\section*{Introduction} \label{Intro}

Point-set matching describes the task of finding an alignment between two sets of points. The problem appears in computer vision applications such as point-set registration \cite{leordeanu_spectral_2005}, object recognition \cite{berg_shape_2005}, and multiple object tracking \cite{wen_multiple_2014}. Often the point-sets are modeled via \textit{graphs}, abstract mathematical objects in which points are represented as vertices and edges define relationships between pairs of vertices. 

User defined attributes characterize the vertices and edges, such as coordinate positions or shape descriptions and lengths of chords connecting vertices, respectively. Specified attributes give insight to observable relationships between vertices and allow for structural analyses of graphs. Graph matching is the optimization problem defined by the search for a correspondence of vertices between a pair of attributed graphs. The optimization problem uses binary variables $x_{ij}$ to specify a matching between vertex \textit{i} in the first graph to vertex \textit{j} of the second. The graph matching domain consists of assignment matrices of size $n_1 \times n_2$, for matching graphs of size $n_1$ and $n_2$. 

\begin{equation}
\label{eqn:Pi}
    \mathcal{X} = \{X \in \{0,1\}^{n_1 \times n_2}: \forall j, \sum_{i = 1}^{n_1} x_{ij} \leq 1, \forall i \sum_{j = 1}^{n_2} x_{ij} = 1\}
\end{equation}

The space $\mathcal{X}$ (Eq~\ref{eqn:Pi}) comprises assignment matrices which each describe a one-to-one alignment between nodes of the two graphs. The specification of the graph matching optimization objective function allows for joint assignment costs: i.e., how the assignment of a pair of vertex-to-vertex assignments changes the quality of the match. Let \textbf{C} be an $n_1 \times n_2$ matrix and \textbf{D} be a $n_1 \times n_2 \times n_1 \times n_2$ tensor storing the vertex-to-vertex and edge-to-edge dissimilarities, respectively. The graph matching optimization problem is expressed in Eq~\ref{eqn:qap}, which takes the form of the quadratic assignment problem (QAP). 

\begin{equation}
\begin{aligned}
& \underset{X \in \mathcal{X}}{\text{minimize}} 
& & \sum_{i=1}^{n_1} \sum_{j=1}^{n_2} \sum_{k=1}^{n_1} \sum_{l=1}^{n_2} d_{ijkl} x_{ij} x_{kl} + \sum_{i=1}^{n_1} \sum_{j=1}^{n_2} c_{ij} x_{ij} \\
\end{aligned}
\label{eqn:qap}
\end{equation}

Graphs are limited in their expressive power as edges can only relate pairs of vertices; hypergraphs extend the definition of a graph to include hyperedges which can specify relationships among an arbitrary number of vertices. Hypergraph matching then concerns finding an optimal vertex correspondence between pairs of attributed hypergraphs. The number of vertices aligned by the most comprehensive hyperedge defines the degree of a hypergraph.

Maximum degree hypergraphs with hyperedges composed of all $n_1$ vertices yield the most comprehensive point-set matching function possible. The optimization objective function captures the dissimilarity arising between the matching: $(l_1, l_2, \dots, l_{n_1}) \mapsto (l'_1, l'_2, l'_3, \dots, l'_{n_1})$. Then, for a given assignment matrix $X \in \mathcal{X}$, the hypergraph matching objective can be expressed using $n_1$ dissimilarity tensors of dimension $2, 4, \dots, 2d, \dots, 2n_1$, each measuring dissimilarity between degree $d$ hyperedges, respectively. Define $\mathbf{Z}^{(d)}$ as the tensor mapping the dissimilarity for the degree \textit{d} hyperedges. The hypergraph matching objective is expressed in Eq~\ref{eqn:additive}.

\begin{multline}
    \mathit{f}(X | \mathbf{Z}^{(1)}, \mathbf{Z}^{(2)}, \dots,  \mathbf{Z}^{(n_1)}) = \sum_{l_1=1}^{n_1} \sum_{l'_1=1}^{n_2} \mathbf{Z}^{(1)}_{l_1 l'_1} x_{l_1 l'_1} + \sum_{l_1=1}^{n_1} \sum_{l'_1=1}^{n_2} \sum_{l_2=l_1+1}^{n_1} \sum_{l'_2=1}^{n_2} \mathbf{Z}^{(2)}_{l_1 l'_1 l_2 l'_2} x_{l_1 l'_1} x_{l_2 l'_2}  \\ + \sum_{l_1=1}^{n_1} \sum_{l'_1=1}^{n_2} \sum_{l_2=l_1+1}^{n_1} \sum_{l'_2=1}^{n_2} \sum_{l_3=l_2+1}^{n_1} \sum_{l'_3=1}^{n_2} \mathbf{Z}^{(3)}_{l_1 l'_1 l_2 l'_2 l_3 l'_3} x_{l_1 l'_1} x_{l_2 l'_2} x_{l_3 l'_3}  + ... \\ + \sum_{l_1=1}^{n_1} \sum_{l'_1=1}^{n_2}  ... \sum_{l_{n_1}=l_{n_1-1}+1}^{n_1} \sum_{l'_{n_1}=1}^{n_2} \mathbf{Z}^{(n_1)}_{l_1 l'_1 \dots l_{n_1} l'_{n_1}} x_{l_1 l'_1} \dots x_{l_{n_1} l'_{n_1}}
\end{multline}

Hypergraph matching allows for the modeling of intricate point-set matching problems through high multiplicity assignment objective function formulations. The $\mathbf{Z}^{(d)}$ dissimilarity terms measure degree \textit{d} hyperedge dissimilarity comprising \textit{d} simultaneous vertex assignments. The range in assignment problem objective complexity from \textit{d}=1 to \textit{d}=$n_1$ trades off model capacity for increased computation. The traditional linear assignment problem (\textit{d}=1) is solvable in polynomial time \cite{kuhn_hungarian_1955}, but treats points between sets independently. Existing graphical methods (\textit{d}=2) and hypergraphical methods (\textit{d}$>$2) rely on approximate searches and do not generalize to high degree formulations of Eq~\ref{eqn:additive}. \textit{Exact Hypergraph Matching} (\textit{EHGM}) is able to find globally optimal solutions to hypergraph matching problems of arbitrary degree, allowing for the modeling of intricate point-set matching tasks. 

\subsection*{Related Research}

Finding an exact solution to the QAP is an $\mathcal{NP}$-hard problem. That is, unless \textit{P}=\textit{NP}, there does not exist a polynomial time solution to exactly solve the QAP \cite{sahni_p-complete_1974}. Higher order assignment problems (i.e. hypergraph matching) are also $\mathcal{NP}$-hard as they are at least as hard as the QAP \cite{pardalos_handbook_2013}. As a result, recent methods for graph matching and lower-degree hypergraph matching focus on heuristic solutions which offer no guarantee on performance \cite{zhou_factorized_2016,zhang_kergm_2019, duchenne_tensor-based_2010,lee_hyper-graph_2011,chertok_efficient_2010}. Heuristic hypergraph matching methods are adapted from existing graph matching algorithms. In particular, spectral methods for solving graph matching (Eq~\ref{eqn:qap}) have been extended to solve hypergraph matching. Duchenne et al. \cite{duchenne_tensor-based_2010} adapt Leordeanu's \cite{leordeanu_spectral_2005} work to obtain a rank-1 approximation of the affinity tensor via higher order power iteration. However, calculating affinity tensors ($\mathbf{Z}^{(d)}$ terms) is computationally prohibitive, especially for higher degree hypergraphs due to the exponentially growing number of entries in the tensors. Simplifying assumptions such as super-symmetry and sparseness are used with sampling methods to build large affinity tensors \cite{duchenne_tensor-based_2010, zaslavskiy_path_2009}. Chertok and Keller propose similar methodology to \cite{duchenne_tensor-based_2010}, but instead unfold the affinity tensor and use the leading left singular vector to approximate the adjacency matrix \cite{chertok_efficient_2010}. All such methods operate outside the permutation matrix space. The Hungarian algorithm or similar binarization step is used to yield a valid assignment, e.g. as in \cite{leordeanu_spectral_2005}. 

Exactness allows for a more rigorous analysis of a hypergraphical point-set matching model than is possible using heuristic techniques. The guarantee of a globally optimal correspondence allows an iterative tuning of the underlying model in pursuit of accurate characterization, whereas the output of a heuristic algorithm could be incorrect due either to the stochasticity of the search or to inadequacy of the optimization objective. Branch-and-bound is a paradigm originally developed to exactly solve the the travelling salesman problem, a type of QAP \cite{land_automatic_1960,little_algorithm_1963}. Branch-and-bound methods recursively commit partial assignments and solve successive subproblems within $\mathcal{X}$. The paradigm iteratively partitions the search space while bounding the optimum at each branch. At each step the method prunes branches which cannot contain lead to the optimum. Convergence occurs when only feasible assignments achieving a global optimum remain. The $\mathcal{NP}-$hardness of the QAP implies convergence occurs only after implicit enumeration of $\mathcal{X}$. 

\subsection{Overview of \textit{EHGM} \& Application to \textit{C. elegans}}

\textit{EHGM} deviates from recent graph matching and hypergraph matching methodology as an exact method, guaranteeing convergence to a globally optimal solution (S1:\textit{Convergence of EHGM}). Heuristic hypergraph matching methods approximate the assignment matrix using the dissimilarity tensor \cite{duchenne_tensor-based_2010,chertok_efficient_2010} whereas \textit{EHGM} builds upon the seminal branch-and-bound algorithm \cite{land_automatic_1960}. \textit{EHGM} extends the methodology to branch and prune based upon a given hypergraphical model. A \textit{k}-tuple of nodes at branch \textit{m} are greedily selected while another step encapsulates the full hypergraphical objective upon selection. These changes enable flexibility in altering the hypergraph matching objective, particularly in allowing for high degree hypergraphical modeling. 

\textit{EHGM} is applied to model \textit{posture} in embryonic \textit{Caenorhabditis elegans} (\textit{C. elegans}), a small, free-living roundworm. The nematode features approximately 550 cells upon hatching, including a set of twenty \textit{seam cells} and two associated neuroblasts. The seam cells and neuroblasts form in lateral pairs along the left and right sides of the worm, resulting in eleven pairs upon hatching \cite{sulston_embryonic_1983}. The neuroblasts appear in the final hours of development, just prior to hatching. The pairs of cells are named, posterior to anterior: \textit{T}, \textit{V6}, \textit{V5}, \textit{Q} (neuroblasts), \textit{V4}, \textit{V3}, \textit{V2}, \textit{V1}, \textit{H2}, \textit{H1}, and \textit{H0}. Each pair's left and right cell is named accordingly; for example, \textit{H1L} and \textit{H1R} comprise the \textit{H1} pair.  Fig~\ref{fig:twist_straight_3d}-A depicts center points of seam cell nuclei located in an example image volume as imaged in the eggshell (left) and straightened to reveal the bilateral symmetry in seam cell locations (right). Fig~\ref{fig:twist_straight_3d}-B shows four sequential images of an embryo, five minutes between images.

\begin{figure}[!ht]
\centering
\includegraphics[width=\textwidth]{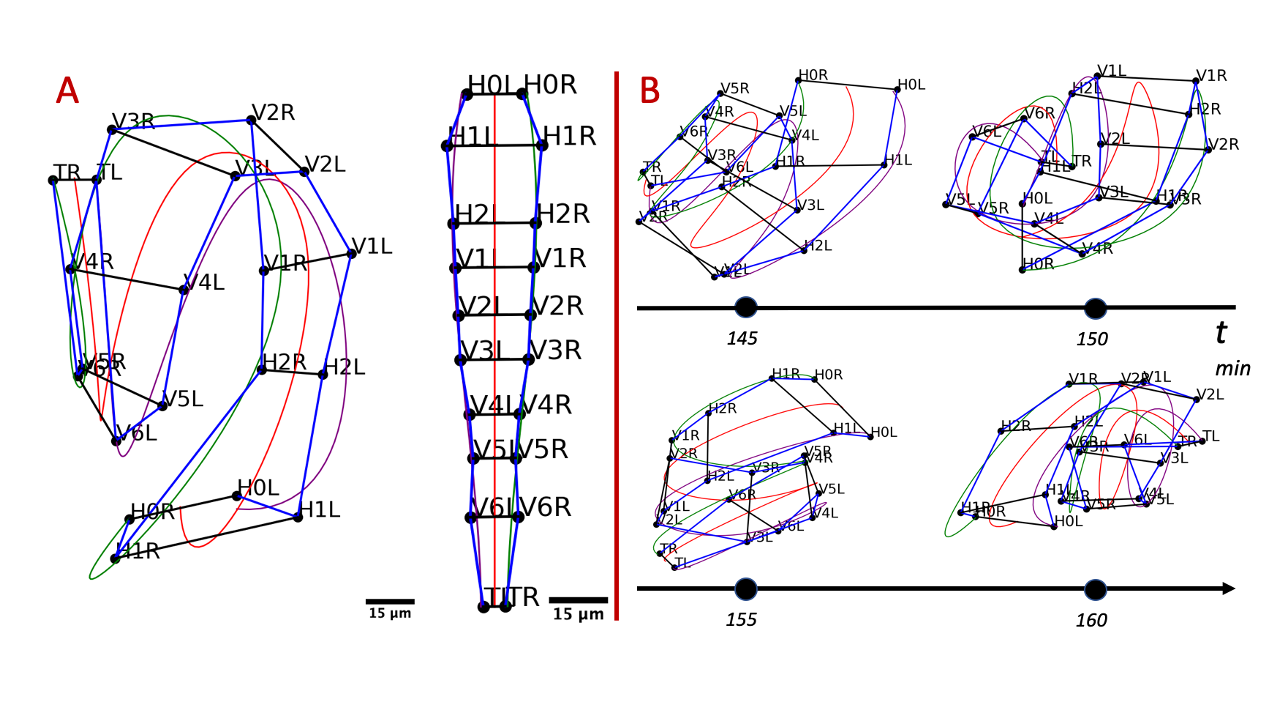}
\caption{\textbf{High spatial resolution, low temporal resolution imaging necessitates posture identification.} A: Manually identified and seam cell nuclei from an imaged \textit{C. elegans} embryo. The cells form in pairs; they are labelled posterior to anterior: \textit{T}, \textit{V6}, ..., \textit{H0}. The identification of all seam cells reveals the embryo's posture. Natural cubic splines through the left and right-side seam cells estimate the coiled body. The left image depicts identified nuclei connected to outline the embryonic worm. The fit splines are used to \textit{untwist} the worm, generating the remapped straightened points in the diagram on the right. B: Labelled nuclear coordinates from a sequence of four images. The embryo repositions in the five minute intervals between images, causing failure of traditional tracking approaches.}
\label{fig:twist_straight_3d}
\end{figure}

We define \textit{posture} as the identification of all seam cells and neuroblasts, which together reveal the shape of the coiled embryo. Posture identification allows for traditional frame-to-frame tracking of imaged cells belonging to various tissues such as the gut, nerve ring, and bands of muscle  \cite{christensen_untwisting_2015}. Images are captured in five minute intervals (Fig~\ref{fig:twist_straight_3d}-B) in order to achieve necessary resolution to track cells of other tissues without disturbing embryo development. Fig~\ref{fig:untwist_track}-A highlights muscle cell nuclei (red dots) with the identified seam cells to contextualize the embryo's positioning. The posture is used to remap the muscle cells such that traditional cell tracking approaches can be applied in the late-stage embryo (Fig~\ref{fig:untwist_track}-B). Fig~\ref{fig:untwist_track}-C depicts the cell remapping process \cite{christensen_untwisting_2015}. The muscle cells are remapped according to splines fitted to the posture. The \textit{untwisted} cell positions are then tracked frame-to-frame (Fig~\ref{fig:untwist_track}-D).

\begin{figure}[!ht]
\centering
\includegraphics[width=\textwidth]{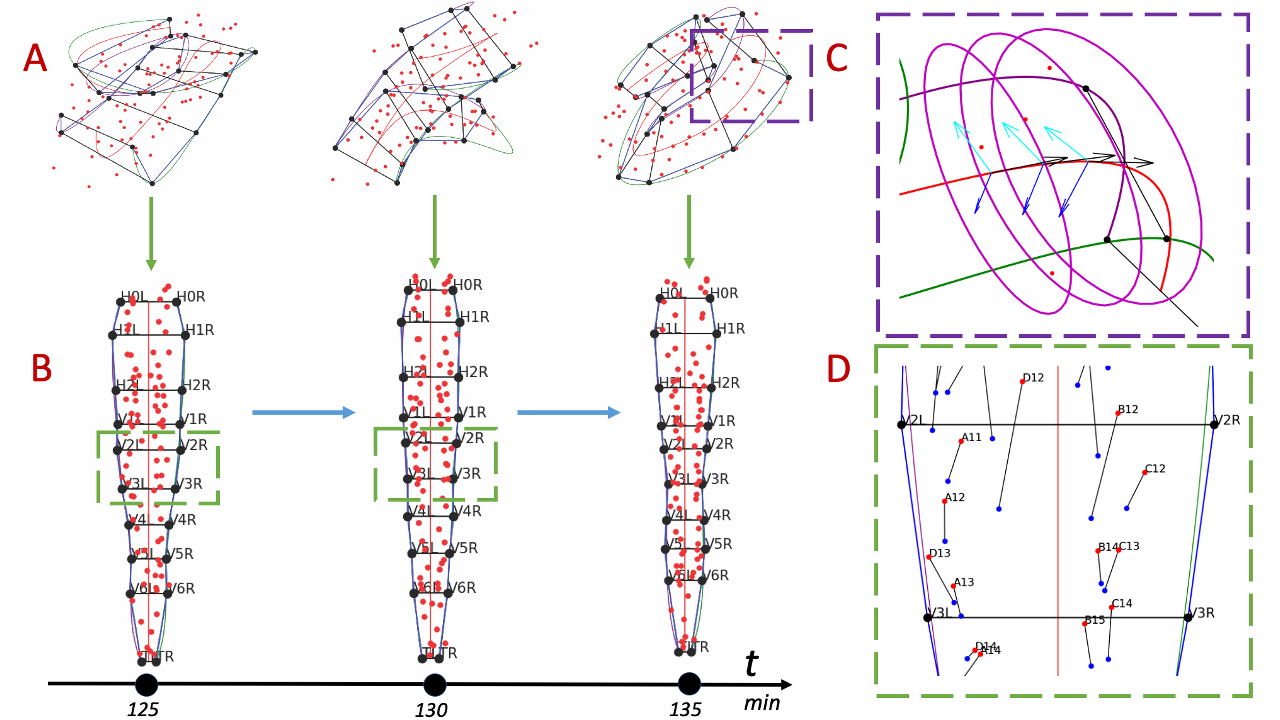}
\caption{\textbf{Posture identification allows the tracking of other cells during late-stage embrygenesis.} A: Seam cell nuclei coordinates (black) and muscle nuclei coordinates (red) in a sequence of three sequential volumetric images. The untwisting process (green arrows) uses the seam cells to remap muscle coordinates to a common frame of reference. B: The remapped muscle nuclei are tracked frame-to-frame (blue arrows). C: A higher magnification view from the right coordinate plot of A. The left, right, and midpoint splines are used to create a change of basis defined by the tangent (black), normal (blue), and binormal (cyan) vectors. Ellipses are inscribed along the tangent of the midpoint spline, approximating the skin of the coiled embryo. D: A portion of the left (red) and center (blue) remapped muscle coordinates. Black lines connect the coordinates, frame-to-frame.}
\label{fig:untwist_track}
\end{figure}

Current methods for posture identification rely on trained users to manually annotate the imaged nuclei using a 3D rendering tool \cite{mcauliffe_medical_2001}. The process takes several minutes per image volume and must be performed on approximately 100 image volumes per embryo \cite{christensen_untwisting_2015}. Manual annotation strategies motivated us to develop \textit{EHGM}, as established methods for point-set matching fail to adequately capture the relationships between seam cells throughout myriad twists and deformations of the developing embryo. Fig~\ref{fig:model_prog-MIPAV} depicts manually identified postures in the first two successive image volumes of Fig~\ref{fig:twist_straight_3d}-B. Manual identification is performed in Medical Imaging, Processing, Analysis and Visualization (MIPAV), a 3D rendering program used for manual annotation \cite{mcauliffe_medical_2001}. 

\begin{figure}[!ht]
\centering
\includegraphics[width=\textwidth]{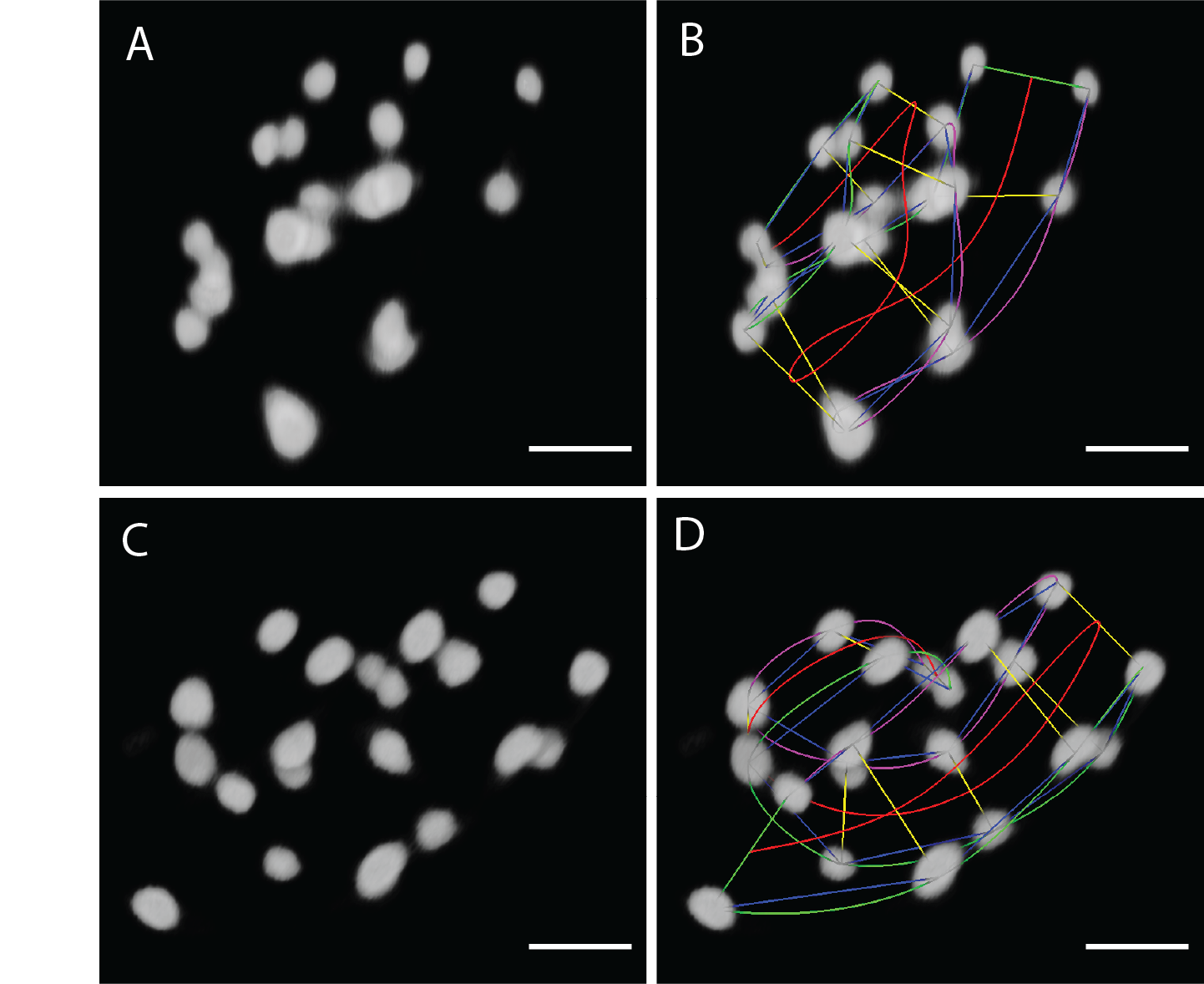}
\caption{\textbf{Manual posture identification in two successive image volumes of Fig~\ref{fig:twist_straight_3d}-B using MIPAV.} The 20 fluorescently imaged seam cell nuclei rendered in two successive image volumes. Scale bar: 10 $\mu m$. A \& B: Seam cell nuclei appearing in two successive image volumes visualized in MIPAV. The five minute interval allows the embryo to reposition between images, yielding entirely different postures. C \& D: Manual seam cell identification by trained users reveals the posture. The curved lines are cubic splines as described in Fig~\ref{fig:untwist_track}-C.}
\label{fig:model_prog-MIPAV}
\end{figure}

\textit{EHGM} uses hypergraphical models comprising biologically driven geometric features to more accurately identify posture than established graphical methods. The limited expressive power of graphical models hinders accurate seam cell identification; graphical models accurately identify posture in 27\% of samples compared to 56\% using a hypergraphical model. User labelling of the posterior-most seam cell nuclei improves the success of hypergraph matching to correctly identifying all nuclei in 77\% of samples. The improved accuracy in posture identification attributed to high-degree hypergraphical modeling solved via \textit{EHGM} paves a path toward automatic posture identification while presenting a general framework for approaching similarly challenging point-set matching tasks. 

\section*{Results} \label{Results}

\subsection{Posture Identification Models}

Posture was predicted via \textit{EHGM} according to three models: a graphical model, denoted \textit{Sides}, and two hypergraphical models. The two hypergraphical models, \textit{Pairs} and \textit{Posture}, showcase \textit{EHGM} as existing algorithms cannot find solutions under such high degree hypergraphs. Each of the three models incrementally use higher degree terms to describe posture. \textit{Sides} follows the form of Eq~\ref{eqn:qap} and leverages pairwise assignments to calculate lengths and widths of portions of the embryo. \textit{Pairs} uses degrees four and six hyperedges to better model local regions of the embryo than is possible with graphical methods which rely on pairwise relationships. \textit{Posture} further demonstrates the capabilities of \textit{EHGM} by including a degree $n_1$ hyperedge to maximize context in evaluating a hypothesized posture. Geometric features such as pair-to-pair rotation angles and left-right flexion angles were developed to more accurately measure and compare posture hypotheses. The calculation of each angle or distance requires identification of multiple seam cells in tandem to calculate, necessitating the use of hyperedges. 

Fig~\ref{fig:model_prog} demonstrates four types of models applied to perform posture identification on the first two sampled images in Fig~\ref{fig:twist_straight_3d}-B. Linear models (Fig~\ref{fig:model_prog}-A \& Fig~\ref{fig:model_prog}-B) are ill-equipped to identify posture due to the repositioning of the embryo between successive images, so linear models are not evaluated on sampled data. The graphical model \textit{Sides} (Fig~\ref{fig:model_prog}-C \& Fig~\ref{fig:model_prog}-D) associates local seam cells via edges (purple). Edge-wise features such as lengths and widths vary if the embryo coils tightly, but are otherwise approximately static frame-to-frame. However, the similarity in these measurements throughout the embryo yields a model incapable of differentiating portions of the embryo. Hypergraphical models \textit{Pairs} (Fig~\ref{fig:model_prog}-E \& Fig~\ref{fig:model_prog}-F) and \textit{Posture} (Fig~\ref{fig:model_prog}-G \& Fig~\ref{fig:model_prog}-H) use aforementioned hyperedges to more strongly characterize embryonic posture.

\begin{figure}[!ht]
\centering
\includegraphics[width=\textwidth]{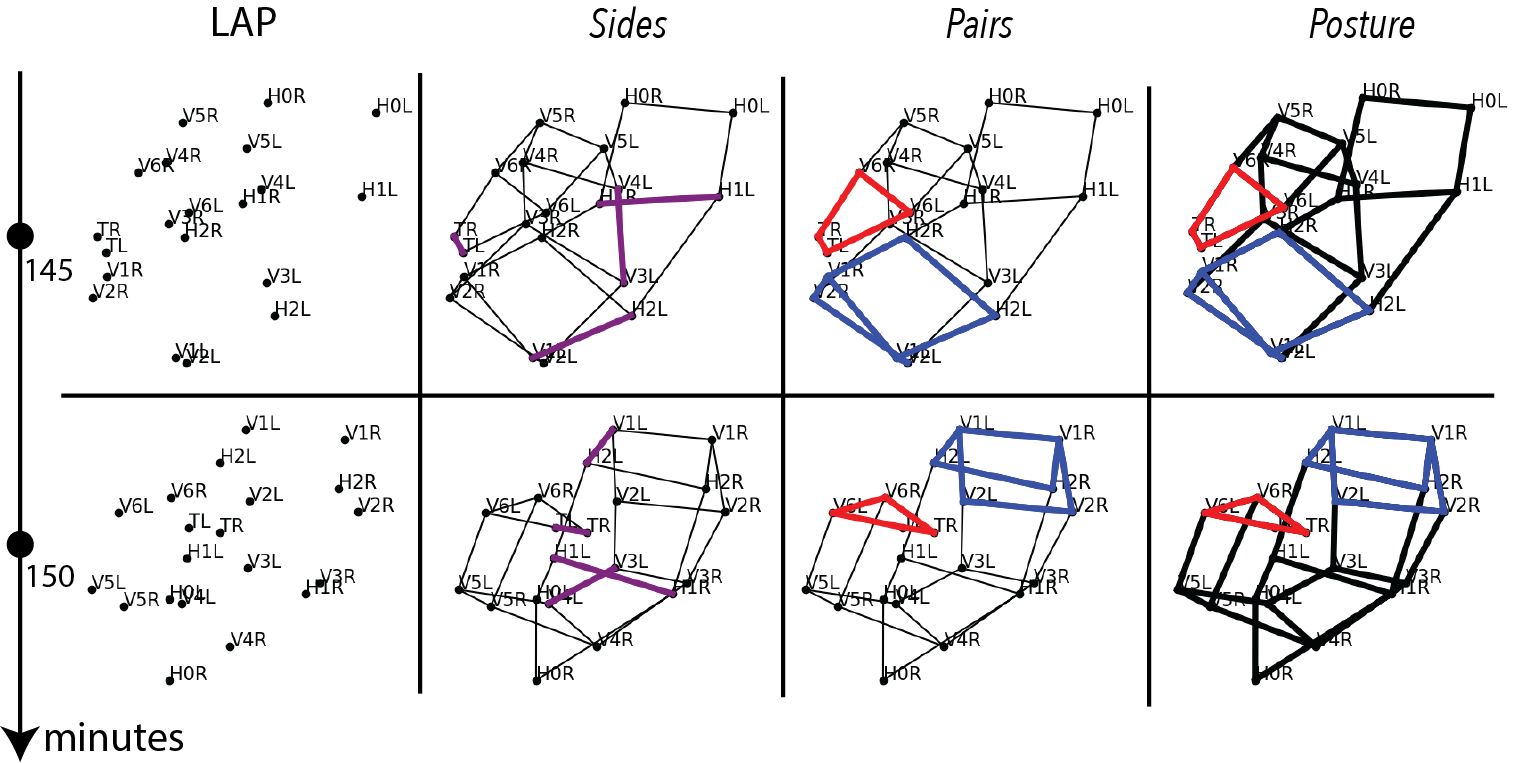}
\caption{\textbf{Posture identification applied to the two successive images in Fig~\ref{fig:model_prog-MIPAV} according to a series of increasingly intricate models.} The embryo repositions between images. A \& B: Linear models (LAP) cannot quantify relationships between seam cells; posture identification is impossible without context of neighboring cell identities. C \& D: A graphical model (\textit{Sides}) specifies edges (purple) between pairs of seam cell nuclei. Edge lengths are relatively static frame-to-frame, but the similarity of edge lengths throughout the embryo causes the edges to have a weak signal in identifying seam cells. E \& F: The \textit{Pairs} model uses degrees four (red) and six (blue) hyperedges to model a greater local context than is possible in a graphical model. G \& H: The \textit{Posture} model extends the \textit{Pairs} model to use a degree $n_1$ (black) hyperedge to evaluate all seam cell assignments jointly.}
\label{fig:model_prog}
\end{figure}

\subsection{Posture Identification Accuracy}

Annotators curated a dataset of seam cell nuclei center coordinates from 16 imaged embryos. Each imaged embryo yielded approximately 80 image volumes for a total of \textit{N}=1264 labelled seam cell nuclei coordinate sets. Homogeneity in \textit{C. elegans} embryo development allowed use of samples spanning multiple embryos to fit models via a leave-one-out approach (S1:\textit{Model Fitting}, S1: \textit{Posture Modeling}). \textit{EHGM} allows for known correspondences, henceforth referred to as \textit{seeds}, to be given as input prior to search initialization. The algorithm was evaluated both in a traditional point-set matching scenario given no \textit{a priori} information, and in a series of seeded simulations. Seeded trials assumed incrementally more pairs given sequentially from the tail pair, \textit{T}, to the fourth pair, \textit{V4} (or \textit{Q} for $n_1$=22 samples). \textit{KerGM} \cite{zhang_kergm_2019}, a leading algorithm for heuristic graph matching, was applied to posture identification. The algorithm used the same connectivity matrix as \textit{Sides}, but processed results frame-to-frame serially, relying on the correct posture identification at the prior image as input to search.

\textit{EHGM} is able to store complete assignments encountered during the search as it compares against the current solution at the final branch. This allowed for an analysis of the similarity between cost minimizing posture hypotheses and progressively higher cost solutions encountered during search. The top \textit{x} accuracy describes the percentage of all \textit{N} samples in which \textit{EHGM} returned the correct posture in the \textit{x} lowest cost solutions; i.e. the top 1 accuracy describes the percentage of samples in which the correct posture was returned as the cost minimizing posture, and the top 3 accuracy is the percentage of samples in which the correct assignment was among 3 lowest cost posture hypotheses returned by the search. Top \textit{x} accuracies are reported alongside the median runtime and the median cost ratio. The cost ratio is defined as the ratio of the correct posture's objective to the cost minimizing posture's objective. A cost ratio greater than one implies the objective of the hypothesized posture is lower than that of the correct posture, suggesting the model is not aptly characterizing posture as an incorrect posture hypothesis was preferred by the model. 

Table \ref{fig:auto_full} shows the percentage of all \textit{N} samples in which the correct posture (correct identification of \textit{all} seam cells) was returned as the minimizer according to \textit{KerGM} and each of the models solved via \textit{EHGM}: \textit{Sides}, \textit{Pairs}, and \textit{Posture}. \textit{KerGM} identified 27\% of sampled postures correctly, outperforming \textit{Sides} (10\%). \textit{Pairs} and \textit{Posture} more effectively identified posture with 52\% and 56\% top 1 accuracies, respectively. Both hypergraphical models also reported a median cost ratio of 1.00, compared to 1.28 of \textit{Sides}, suggesting the hypergraphical representations of coiled posture provided enhanced discriminatory power across samples. The hypergraphical models demonstrated small trade-offs between accuracy and runtime. The \textit{Posture} model's $n_1$ degree hypergraphical features improved accuracy over \textit{Pairs}, 56\% to 52\%, in exchange for longer median runtime, 60 minutes to 43 minutes. Differences between the top 1 and top 3 accuracies reflect the challenge in posture identification. The optimums under the \textit{Pairs} and \textit{Posture} models were often similar to those of similar posture hypotheses. Notably, the \textit{Posture} model returned the correct posture in the top 3 hypotheses in approximately 67\% of samples, an approximate 20\% increase in relative accuracy over the top 1 percentage, 56\%. 

\begin{table}[!ht]
\centering

\begin{tabular}{|l|c|c|c|c|c|c|c|c|c|c|c|}
\toprule
{} &  Top 1 (\%) &  Top 2 (\%) &  Top 3 (\%)  &  Top 5 (\%)  &  Top 10 (\%)  &  R (minutes) &  CR \\
\midrule
\textit{KerGM} &    27 &     27 &     27 &     27 &       27 &   .01 &                \\ \hline
\textit{Sides}   &    10 &    14 &    15 &    16 &      16 &  5.97 &               1.28 \\ \hline
\textit{Pairs} &    52 &    60 &    63 &    65 &      65 & 43.22 &               1.00 \\ \hline
\textit{Posture}   &    56 &    65 &    67 &    68 &      68 & 60.35 &               1.00 \\ 
\bottomrule
\end{tabular}

\caption{\textbf{Hypergraphical model \textit{Posture} achieves highest accuracy.} Posture identification accuracies across all \textit{N}=1264 samples. \textit{KerGM} is compared to proposed models. The first columns list the top \textit{x} accuracy as a percentage of samples. The column titled \textit{R} shows the median runtime of each model in minutes. \textit{CR} reports the median cost ratio, defined as the ratio of the correct posture cost to the returned posture cost.}
\label{fig:auto_full}
\end{table}

Posture identification results were stratified by the presence of the \textit{Q} neuroblasts; 875 of the 1264 samples contain only the seam cells while the remaining 389 samples are mature enough to have the \textit{Q} neuroblasts. Table \ref{fig:auto_20_22} depicts the findings presented in Table \ref{fig:auto_full} split by \textit{Q} neuroblast presence. \textit{KerGM} and all models solved via \textit{EHGM} achieved a higher accuracy on \textit{Q} samples. Notably, the \textit{Posture} model's top 3 accuracy is higher on the \textit{Q} samples (82\%) than the pre-\textit{Q} samples (60\%). The extra pair of coordinates provided substantial context, further defining the coiled shape and helping to penalize incorrect postures. 

\begin{table}[!ht]
\centering

\begin{tabular}{|l|c|c|c|c|c|c|c|c|c|c|c|}
\toprule
{} &  Top 1 (\%) &  Top 2 (\%) &  Top 3 (\%)  &  Top 5 (\%)  &  Top 10 (\%)  &  R (minutes) &  CR \\
\midrule
\textit{KerGM} &     25 &      25 &      25 &      25 &       25 &     .01 &       \\ \hline
\textit{Sides}   &      7 &     10 &     11 &     12 &      12 &    4.81 &  1.36 \\ \hline
\textit{Pairs} &     44 &     51 &     55 &     57 &      58 &   34.25 &  1.04 \\ \hline
\textit{Posture}   &     48 &     57 &     60 &     61 &      62 &   51.12 &  1.00 \\ 
\bottomrule
\end{tabular}

\begin{tabular}{|l|c|c|c|c|c|c|c|c|c|c|c|}
\toprule
{} &  Top 1 (\%) &  Top 2 (\%) &  Top 3 (\%)  &  Top 5 (\%)  &  Top 10 (\%)  &  R (minutes) &  CR \\
\midrule
\textit{KerGM} &     35 &      35 &      35 &      35 &       35 &        .01 & \\ \hline
\textit{Sides}   &     19 &     25 &     26 &     26 &      26 &       9.66 &  1.16 \\ \hline
\textit{Pairs} &     71 &     80 &     82 &     82 &      82 &      56.58 &  1.00 \\ \hline
\textit{Posture}   &     72 &     81 &     82 &     83 &      83 &      72.60 &  1.00 \\
\bottomrule
\end{tabular}
\caption{\textbf{Hypergraphical models leverage \textit{Q} neuroblasts to identify posture.} The samples are split according to the absence (top) or presence (bottom) of the \textit{Q} neuroblasts, which form in the last two hours of development. There are $875$ $n_1$=20 cell samples and $389$ $n_1$=22 \textit{Q} samples. Reported methods more accurately identify embryonic posture in the \textit{Q} samples, suggesting the increased continuity along the body of the embryo allows for more consistent posture identification.}
\label{fig:auto_20_22}
\end{table}

Seeded experiments specifying nuclear identities provided \textit{a priori} information starting with the tail pair, and incrementally identified more pairs in the posterior region. Each experiment was given five minutes of maximum runtime; a semi-automated solution requiring more runtime was deemed infeasible. Top 1 and top 3 accuracy percentages are reported by \textit{EHGM} models and number of seeded pairs in Table \ref{fig:backward_1_3}. Seeding yielded decreasing marginal improvements to accuracy and runtime. Fig~\ref{fig:Pairs_P-F_Q_comparison} depicts top 1 accuracies and median runtimes across seeded experiments for the \textit{Pairs} and \textit{Posture} models split by \textit{Q} pair labelling. Particularly, seeding the first two pairs, \textit{T} and \textit{V6}, greatly reduces the median runtime while also netting the largest gains in top 1 accuracy, partially attributable to \textit{EHGM} converging in the given timeframe. 

\begin{table}[!ht]
\centering

\begin{tabular}{|l|c|c|c|c|c||c|c|c|c|c|}
\toprule
 & \multicolumn{5}{c}{Top 1 (\%)} &    \multicolumn{5}{c}{Top 3 (\%)}  \\
{} &  None &   T  &  T-V6  &  T-V5  &  T-V4  &  None & T  &  T-V6  &  T-V5  &  T-V4   \\
\midrule
\textit{Sides}   &     9 &   10 &    22 &    29 &    37 &    13 &   15 &    27 &    35 &    43  \\ \hline
\textit{Pairs} &    34 &   49 &    72 &    79 &    84 &    38 &   54 &    77 &    83 &    87  \\ \hline
\textit{Posture}   &    25 &   36 &    68 &    79 &    84 &    27 &   39 &    73 &    84 &    87  \\
\bottomrule
\end{tabular}

\caption{\textbf{Seeding posterior pair identities promotes accurate posture identification and reduces runtime.} Top 1 and top 3 seeded posture identification accuracies across all samples. All trials had a five-minute maximum runtime. The rows again correspond to each model. Columns specify which pairs were given as seeds prior to search. The \textit{None} columns recreate the original no information task. The subsequent columns specify which pairs are correctly identified prior to search.}
\label{fig:backward_1_3}
\end{table}

\begin{figure}[!ht]
\centering
\includegraphics[width=\textwidth]{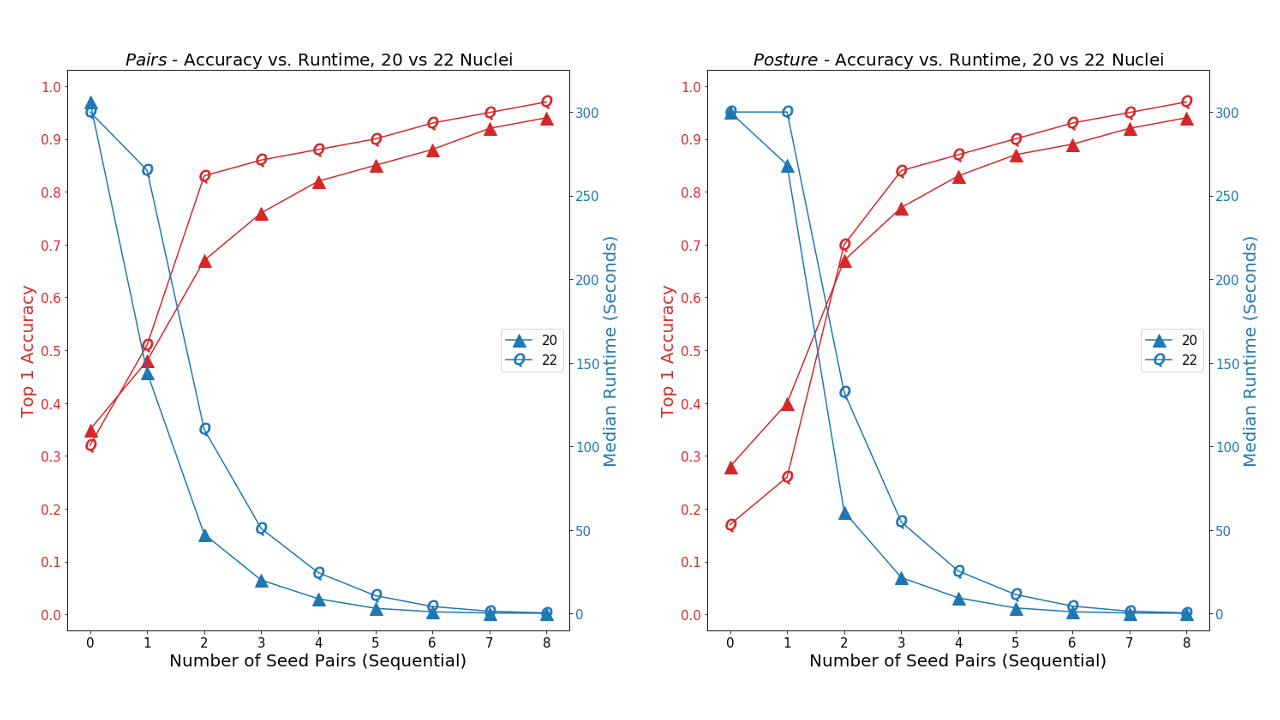}
\caption{\textbf{Evaluating the \textit{Pairs} and \textit{Posture} models as seam cell identities were seeded.} The \textit{Pairs} and \textit{Posture} models top 1 accuracies and median runtimes by Q pair labelling. Posterior pair seeding drastically improved top 1 accuracy and reduced runtime when applying both models. \textit{Q} pair samples required more runtime ($n_1$=22 as opposed to $n_1$=20), but the added context improved posture identification accuracy. The majority of samples converged within 5 minutes when seeded with the \textit{T} and \textit{V6} pairs of nuclei.}
\label{fig:Pairs_P-F_Q_comparison}
\end{figure}

\section*{Discussion} \label{Disc}

We have presented \textit{EHGM} as a dynamic and effective tool for intricate point-set matching tasks. The hypergraph matching algorithm provides a method in which to gauge the efficacy of modeling point correspondences in conservatively-sized problems; problems featuring larger numbers of points likely contain the context required to match adequately via lower degree models. For example, postures in samples containing \textit{Q} nuclei were more accurately identified across models, but the largest marginal gain in accuracy came from \textit{Sides} (\textit{d}=2) to \textit{Pairs} (\textit{d}=4,6). The results suggest that added context throughout the embryo would further improve posture identification accuracy, reducing the reliance on higher degree (and thus more computationally expensive) hypergraphical objective function formulations. \textit{EHGM} specifically addresses a gap in literature concerning challenging point-set matching applications in which domain-specific features lead to rigorously testable models. Seeding allows a wider range of problems to be approached, and mitigates the computational expense of the algorithm for scenarios featuring larger point-sets. 

Posture identification in embryonic \textit{C. elegans} is a challenging problem benefiting from high degree hypergraphical modeling. \textit{EHGM} equipped with biologically inspired hypergraphical models led to substantial improvement in posture identification. The top 1 accuracy doubled from 27\% with a graphical model to 56\% via the \textit{Posture} model (Table~\ref{fig:auto_full}). The top 3 accuracy rate improved to 67\%, highlighting the challenge in precisely specifying the coiled embryo due to the similarity of competing posture hypotheses. The presence of \textit{Q} neuroblasts further contributed to accurate posture identification. The added context empowered the \textit{Posture} model to identify the correct posture in 82\% of \textit{Q} samples (Table~\ref{fig:auto_20_22}.

The top \textit{x} percentage accuracy metric reflects the need to correctly identify \textit{all} seam cells in order to recover the underlying posture, but does not distinguish between hypotheses that are incorrect due to one cell identity swap or a more systemic modeling inadequacy. A qualitative analysis highlighted a few themes among incorrectly predicted postures. The foremost errors concern the tail pair cells, \textit{TL} and \textit{TR}; spurious identifications occurred when the tail pair coiled against another the body of the embryo, causing one tail cell identity to be interchanged with a cell of a nearby body pair. The variance of feature measurements in the posterior region resulted in similar costs for postures with minor differences about the posterior region. 

Pair seeding allows for the strengths of \textit{EHGM} to compensate for the most challenging aspect of posture identification. The posterior region of the embryonic worm is especially flexible and contributes to the majority of reported errors. Feature engineering stands to create hypergraphical models more capable of reliable posture identification, particularly in contextualizing the posterior region. The method and application outline a protocol for challenging point-set matching tasks. 

\section*{Methods}

\subsection*{\textit{Exact Hypergraph Matching}} \label{EHGM}

\textit{EHGM} extends the branch-and-bound paradigm to exactly solve hypergraph matching. The algorithm performs the search in the permutation space $\mathcal{X}$ subject to a given branch size \textit{k} which specifies the number of vertices assigned at each branch. A size $n_1$ hypergraph will require $M \vcentcolon = \frac{n_1}{k}$ branch steps, where branch \textit{m} concerns the assignment of vertices $((m-1)k+1, (m-1)k+2, \dots, mk)$; vertices $1, 2, \dots, mk$ have been assigned upon completion of the $m^{th}$ branch. The set $\mathbf{P}$ contains all possible permutations of the indices of the unordered point set, $|\mathbf{P}| = \frac{n_2!}{(n_2 - k)!}$. $\mathbf{P}$ is incrementally subset into queues $\mathbf{Q}_m \subseteq \mathbf{P}$ at branches $m=1, 2, \dots, M$ at each branching. The queue $\mathbf{Q}_m$ is subset according to both a pruning rule which eliminates permutations leading to a suboptimal solution as well as the one-to-one constraints of $\mathcal{X}$. The search converges to a global optimum upon the implicit enumeration of $\mathbf{Q}_1 = \mathbf{P}$. 

The objective function \textit{f} is further stratified according to the branch size \textit{k}. Lower degree ($d \leq 2k$) hyperedge dissimilarity tensors are computed prior to search. Branches comprising \textit{k}-tuples of vertices are partially assigned in a greedy manner according these lower degree hyperedge dissimilarities via the selection rule \textit{H}. Later branches accrue higher degree ($d > 2k$) hyperedge dissimilarities which are calculated at time of branching; the intent of the method is to rely on lower degree terms to steer the search towards an optimum in effort to minimize the number of branches explored. The aggregation rule \textit{I} accrues higher degree hyperedge dissimilarity terms upon branching, further guiding the pruning step and ensuring the complete specification of the objective $f$. 

The branching and selection rules are designed to reduce computation performed throughout the search. A partial assignment at branch $m$: $\mathbf{K}_m = (l'_{(m-1)k+1}, l'_{(m-1)k+2}, \dots, l'_{mk}) \in \mathbf{Q}_m$ is selected via precomputed lower degree hyperedge dissimilarity tensors $\mathbf{Z}^{(1)}, \dots, \mathbf{Z}^{(2k)}$. A larger branch size \textit{k} results in a selection rule with larger scope of the optimization landscape, better equipped to place optimal branches earlier in each queue $\mathbf{Q}_m$ at time of branching. However, computing the lower degree dissimilarity tensors prior to search can be prohibitively expensive for larger point-sets. 

\subsection*{Selection \& Aggregation} \label{S&A}

The first branch permutation $\mathbf{K}_1 = (l'_1, l'_2, \dots, l'_k) \in \mathbf{Q}_1 = \mathbf{P}$ assigns vertices $(l_1, l_2, \dots, l_k)$ to points $(l'_1, l'_2, \dots, l'_k)$ according to the initial branch selection rule $H_{1}$. Eq~\ref{eqn:H_1} defines a cost given dissimilarity tensors $\mathbf{Z}^{(1)}, \mathbf{Z}^{(2)}, \dots \mathbf{Z}^{(k)}$ according to a permutation $\mathbf{K}_1$. The \textit{k} pairs of constraints given by the branch \textit{m} and permutation of point indices $\mathbf{K}_m$: $\{(l_1, l'_1), \dots, (l_k, l'_k)\}$ enables a simplification in the objective formulation. 

\begin{multline}
\label{eqn:H_1}
    H_{1}(\mathbf{K}_1 | \mathbf{Z}^{(1)}, \mathbf{Z}^{(2)}, ..., \mathbf{Z}^{(k)}) \vcentcolon = \\ \sum_{i_1=1}^{k} \mathbf{Z}^{(1)}_{l_{i_1} l'_{i_1}} + \sum_{i_1=1}^{k} \sum_{i_2=i_1+1}^{k} \mathbf{Z}^{(2)}_{l_{i_1} l'_{i_1} l_{i_2} l'_{i_2}} + ... + \sum_{i_1=1}^{k} \sum_{i_2=i_1+1}^{k} \dots \sum_{i_k=i_{k-1}+1}^{k} \mathbf{Z}^{(k)}_{l_1 l'_{i_1} l_{i_2} l'_{i_2} \dots l_{i_k} l'_{i_k}}
\end{multline}

Subsequent branches $m = 2, 3, \dots M$ then use the general selection rule $H_m$ to order the permutations of the $m^{th}$ branch: $\mathbf{K}_m = (l'_{(m-1)k+1}, l'_{(m-1)k+2}, \dots l'_{mk}) \in \mathbf{Q}_m$. Branch $\mathbf{K}_m$ incurs a selection rule cost $H_m$ according to Eq~\ref{eqn:H} comprising lower degree hyperedge dissimilarities for assignments both within branch $m$ and the assignments between branches $1, 2, \dots, m-1$ and branch $m$. The partial assignment constraints $\mathbf{K}_m$ allow further simplification of notation; the reversed order of summation indices satisfies the criteria that only hyperedge dissimilarities pertaining to branch $m$ assignments are considered via $H_m$.

\begin{multline}
\label{eqn:H}
    H_{m}(\mathbf{K}_{m} | \mathbf{K}_{1}, ..., \mathbf{K}_{m-1}, \mathbf{Z}^{(1)}, ..., \mathbf{Z}^{(2k)}) \vcentcolon = \\
    \sum_{i_1=(m-1)k+1}^{mk} \mathbf{Z}^{(1)}_{l_{i_1} l'_{i_1}} + \sum_{i_2=(m-1)k+1}^{mk} \sum_{i_1=1}^{i_2-1} \mathbf{Z}^{(2)}_{l_{i_1} l'_{i_1} l_{i_2} l'_{i_2}} \\ + \sum_{i_3=(m-1)k+1}^{mk} \sum_{i_2=1}^{i_3-1} \sum_{i_1=1}^{i_2-1} \mathbf{Z}^{(3)}_{l_{i_1} l'_{i_1} l_{i_2} l'_{i_2} l_{i_3} l'_{i_3}} + ... +  \sum_{i_{2k} = (m-1)k+1}^{mk} \sum_{i_{2k-1} = 1}^{i_{2k}-1}  \dots \sum_{i_1 = 1}^{i_2-1} \mathbf{Z}^{(2k)}_{l_{i_1} l'_{i_1} \dots l_{i_{2k}} l'_{i_{2k}}}
\end{multline}

The greedy selection rule orders queues $\mathbf{Q}_m$, but does not account for higher degree ($2k < d \leq n_1$) hyperedge dissimilarities. Precomputing higher degree dissimilarity tensors can be both computationally expensive, and inefficient as ideally only a small percentage of combinations are queried throughout the search. The aggregation rule $I_m$, $m = 3, 4, \dots, M$ measures the dissimilarity attributable to higher degree ($2k < d \leq mk$) hyperedges accessible due to branch $m$ partial assignments. The aggregation rule updates the cost of branch $\mathbf{K}_m$ assignments, further informing the pruning step to subset the next queue $\mathbf{Q}_{m+1}$. The greedy selection rule $H_m$ in tandem with the aggregation rule $I_m$ aim to minimize the total computation performed in finding an optimum. The definition $I_m$ follows from the general selection rule $H_m$, but is applied to the higher degree hyperedge dissimilarities. The aggregation rule $I_m$ (Eq~\ref{eqn:I}) can be expressed as the degree $d$ dissimilarities calculable upon assignments of branch $m$ assignments for degrees $2k < d \leq mk$. 

\begin{equation}
\label{eqn:I}
    I_m(\mathbf{K}_{m} | \mathbf{K}_{1}, \mathbf{K}_{2}, \dots, \mathbf{K}_{m-1}, \mathbf{Z}^{(2k+1)}, \dots, \mathbf{Z}^{(mk)}) \vcentcolon = \sum_{d=2k+1}^{mk} \sum_{i_d=(m-1)k+1}^{mk} \sum_{i_{d-1}=1}^{i_d-1} ... \sum_{i_1=1}^{i_2-1} \mathbf{Z}^{(d)}_{l_{i_1} l'_{i_1} ... l_{i_d} l'_{i_d}}
\end{equation}

The $m^{th}$ branch allows for hyperedge dissimilarities up to degree \textit{mk} concerning the first \textit{mk} assignments. The $M^{th}$ branch yields a complete assignment, allowing the evaluation of maximum degree $n_1$ hyperedge dissimilarities. The partitioning and further regrouping of each $H_m$ and $I_m$ as defined fully accounts for the objective \textit{f} while allowing efficient computation during the search (S1:\textit{Hypergraphical Objective Decomposition}, S1:\textit{Convergence of EHGM}).

\subsection*{Posture Identification in Embryonic \textit{C. elegans}} \label{PI}

\textit{Caenorhabditis elegans} (\textit{C. elegans}) is a small, free-living nematode found across the world. The worm is often studied as a model of nervous system development due to its relative simplicity \cite{white_structure_1986, rapti_perspective_2020}. The adult worm features only 302 neurons, the morphology and synaptic patterning of which have been determined via electron microscopy \cite{white_structure_1986}. The complete embryonic cell lineage has also been determined \cite{sulston_embryonic_1983}; methods and technology have been developed to allow study of cell position and tissue development in the embryo \cite{bao_automated_2006, boyle_acetree_2006, santella_wormguides_2015, mace_high-fidelity_2013, cao_establishment_2020, wang_high-content_2019}. Systems-level studies of these processes may be able to discover larger-scale principles underlying developmental events.

The embryo features a set of twenty \textit{seam cells} and two associated neuroblasts. The seam cells and neuroblasts together describe anatomical structure in the coiled embryo, acting as a type of ``skeleton'' outlining its body. Identification of the seam cells and neuroblasts defines the embryo's posture. Fluorescent proteins are used to label cell nuclei, including the seam cell nuclei so that they may be visualized during imaging, e.g. with light sheet microscopy \cite{wu_spatially_2013}. Volumetric images are captured at five minute intervals in order to capture subcellular resolution without damaging the worm's development \cite{christensen_untwisting_2015}. Seam cell nuclei appear in the fluorescent images as homogeneous spheroids. Their positions relative to other nuclei and other salient cues present in the image volumes comprise the information that trained users employ to manually identify seam cells. Fig~\ref{fig:mipav} shows the two rendered fluorescent images from Fig~\ref{fig:twist_straight_3d}-A in Medical Image Processing, Analysis and Visualization (MIPAV), a 3D rendering tool \cite{mcauliffe_medical_2001}. The interface is used to annotate both seam cell nuclei and track remapped nuclei, as in Fig~\ref{fig:model_prog-MIPAV} \cite{christensen_untwisting_2015}. 

\begin{figure}[!ht]
\centering
\includegraphics[width=\textwidth]{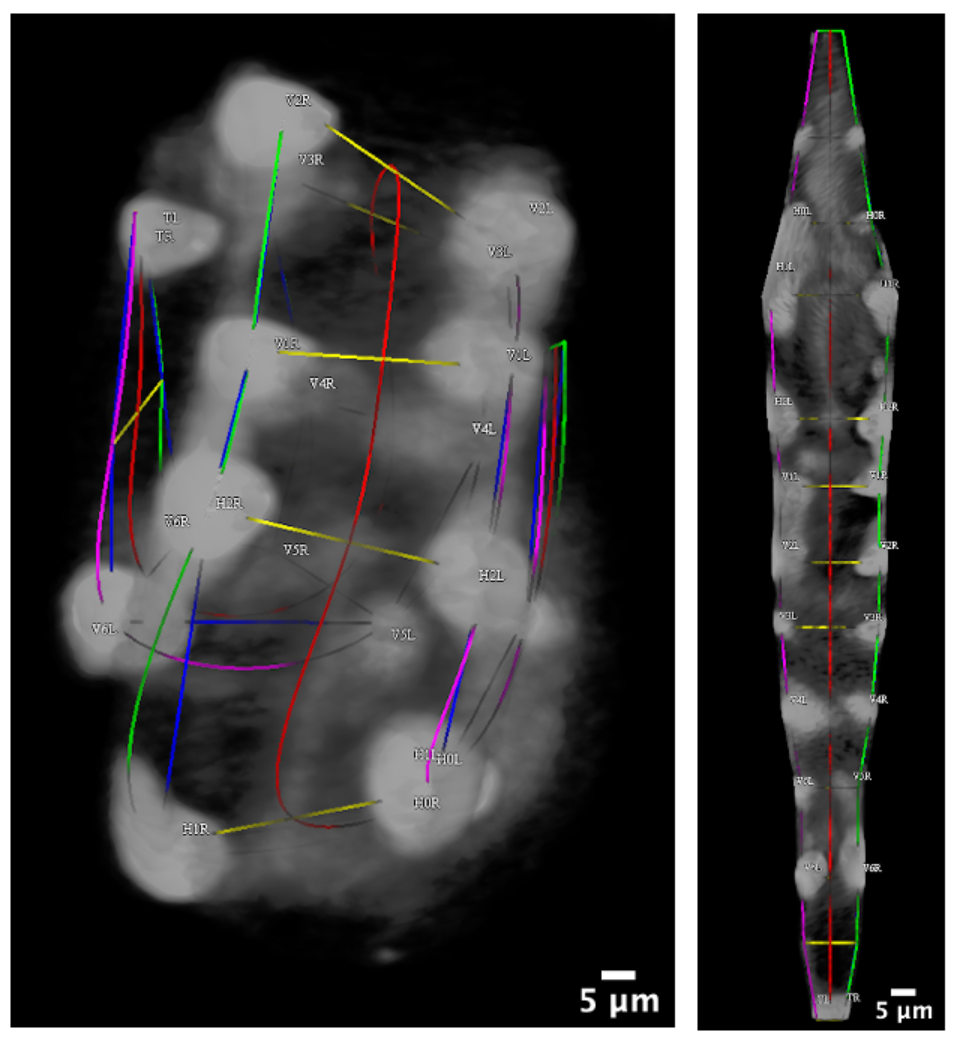}
\caption{\textbf{Rendered image volumes in the MIPAV GUI.} The imaged twisted embryo (left) and imaged straightened embryo (right) rendered in Medical Image Processing, Analysis and Visualization (MIPAV) \cite{mcauliffe_medical_2001}. The fluorescent images are those depicted in Fig~\ref{fig:twist_straight_3d}-A. Trained users navigate the MIPAV GUI to identify seam cells based upon relative positioning and other salient features such as specks of fluorescence on the skin. Correct identification of all imaged nuclei reveals the coiled embryonic posture. Green (left), red (center), and purple (right) splines yield an approximation of the coiled embryo's posture. Yellow lines connect seam cell nuclei laterally. The splines are used with the image volume to sweep planes orthogonal to the center spline, yielding the straightened embryo image.}
\label{fig:mipav}
\end{figure}

We cast posture identification as hypergraph matching and use \textit{EHGM} to solve the resulting optimization problem. The proposed models: \textit{Sides}, \textit{Pairs}, and \textit{Posture} trade off modeling capacity for increased computation to identify optimal solutions. \textit{Sides} expresses posture identification as graph matching; edge-wise (degree \textit{d}=2) features take the form of standardized chord lengths between nuclei laterally and sequentially along each side. The first hypergraphical model, \textit{Pairs}, employs a greater local context than \textit{Sides} using degrees four and six hyperedges to describe relationships between seam cells. Hyperedges formed by two or three sequential pairs (\textit{d}=4,6) better detail local regions throughout the embryo than is capable of a graphical model. Fig~\ref{fig:Pairs_HG}-A presents the hyperedge connectivity among nodes in the \textit{Pairs} model \cite{valdivia_analyzing_2021}. The \textit{Posture} model extends the \textit{Pairs} model by leveraging complete posture (\textit{d}=$n_1$) features in effort to further discriminate between posture hypotheses that appear similar in sequential regions of the embryo. Geometric features help contextualize the coiled posture. Fig~\ref{fig:bends_twists} illustrates three of the features used in the \textit{Pairs} and \textit{Posture} models. The angle $\Theta$ measures the angle between three successive pair midpoints. The angles $\Theta$ decrease throughout development as the worm elongates. Pair-to-pair twist angles $\varphi$ and $\tau$ penalize posture hypotheses in which posterior to anterior transitions are jagged and unnatural in appearance. See S1:\textit{Posture Modeling} for further details and specification of model features. 

\begin{figure}[!ht]
\centering
\includegraphics[width=\textwidth]{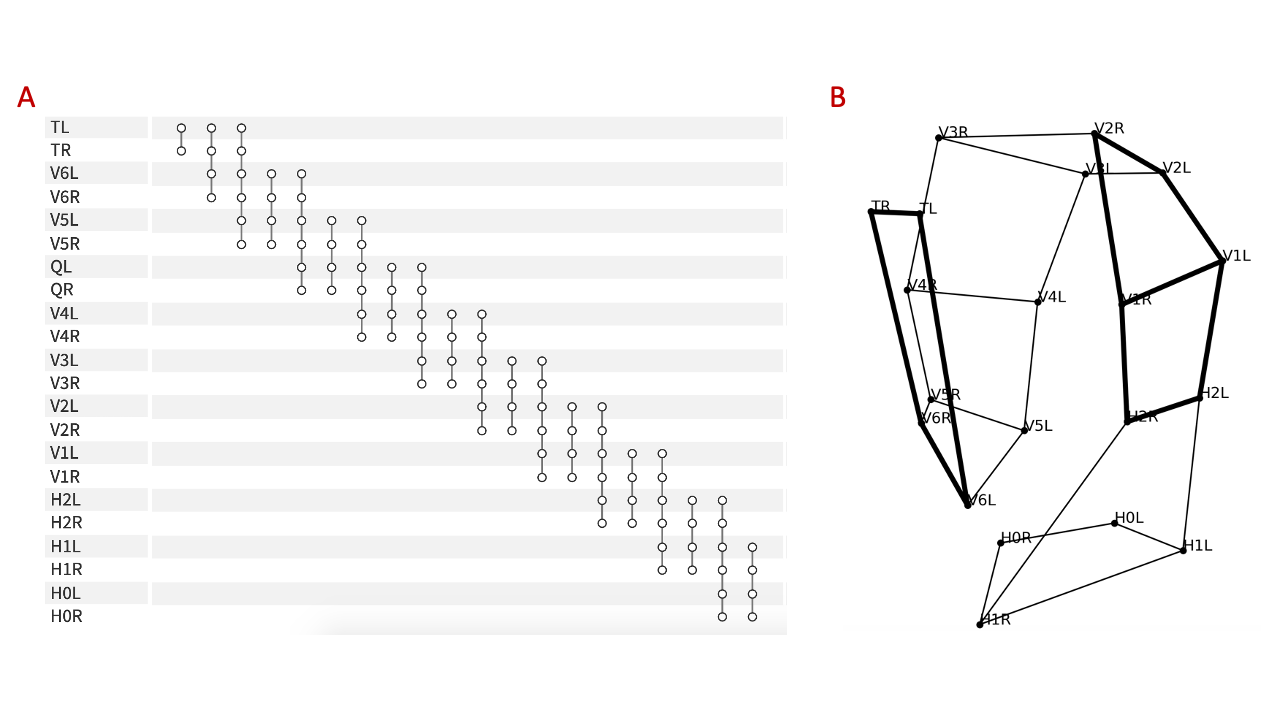}
\caption{\textbf{The \textit{Pairs} hypergraphical model uses expansive local contexts about each portion of the embryo.} A: The \textit{Pairs} hyperedges connect local seam cell nuclei in sets of four and six. B: Degree four hyperedges connect sequential pairs of seam cells while degree six hyperedges connect sequential triplets of pairs. The posterior-most degree four hyperedge and a central degree six hyperedge are \textbf{bolded}.}
\label{fig:Pairs_HG}
\end{figure}

\begin{figure}[!ht]
\centering
\includegraphics[width=\textwidth]{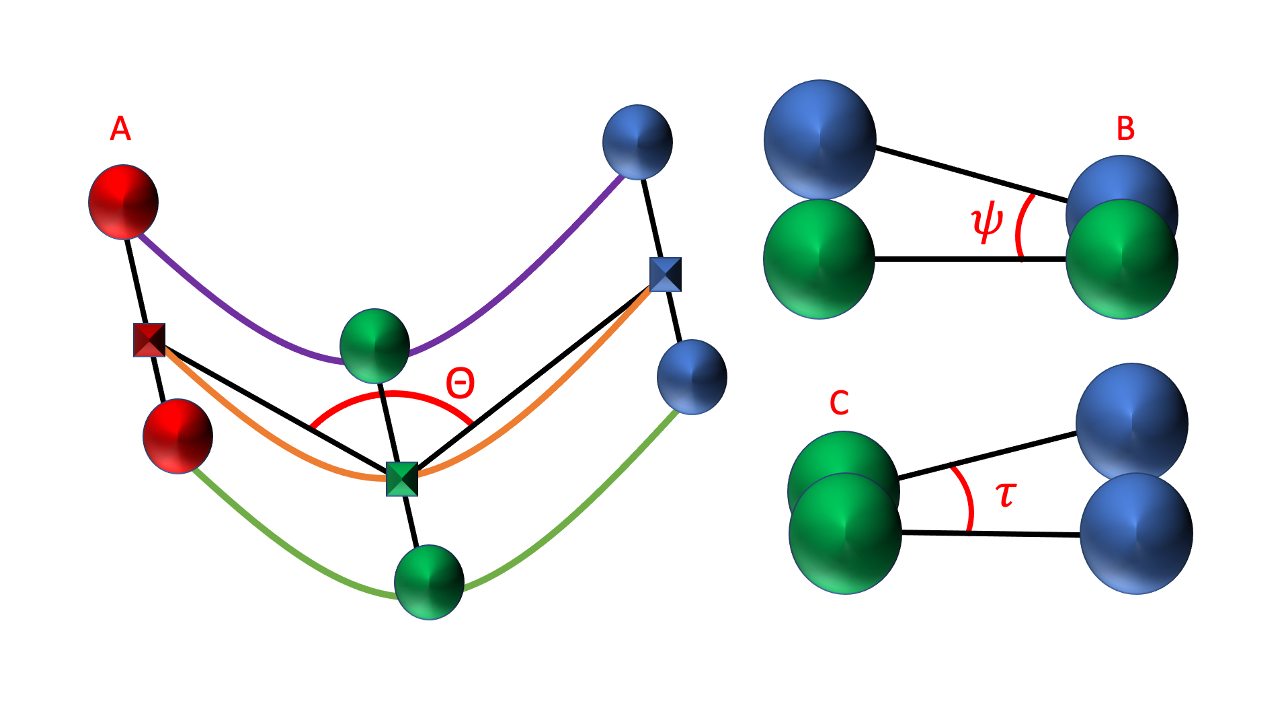}
\caption{\textbf{Hypergraphical geometric features contextualize seam cell assignments.} Anatomically inspired geometric features describe bend and twist of a posture assignment. A: Three pairs of sequential nuclei: red, green, blue. Rectangles represent pair midpoints. The angle $\Theta$ in red is used as a degree six feature given six point to nuclei assignments. B, C: Degree four hypergraphical features measuring twist angles $\varphi$ and $\tau$. These angles measure posterior to anterior twist pair-to-pair and left-right twist, respectively.}
\label{fig:bends_twists}
\end{figure}

The traditional point-set matching task requires a labelled point-set and a second unidentified point-set. Higher order features such as bend and twist angles may vary largely frame-to-frame depending on the posture at moment of imaging. However, elongation throughout late-stage development causes macroscopic trends in these geometric features. We estimate a template posture as a composite of feature measurements from a corpus of manually annotated postures. The templates are time dependent to reflect the elongation from the first point of imaging throughout development until hatching. See S1:\textit{Model Fitting} for details on template estimation.

Together, the fitted models are used with \textit{EHGM} to identify posture in imaged \textit{C. elegans} embryos. The branch size \textit{k}=2 is set for all models, i.e. a lateral pair of seam cell identities are assigned at each branch starting with the tail pair cells \textit{TL} and \textit{TR}. The successive pair cells, \textit{V6L} and \textit{V6R}, are assigned given the established cells and hypergraphical relationships accessible with the hypothesized identities. Fig~\ref{fig:tree} depicts \textit{EHGM} applied to the sample image depicted in Fig~\ref{fig:twist_straight_3d}-A. The initial pair (\textit{TL} and \textit{TR}) is selected, instantiating a search tree (Fig~\ref{fig:tree}-A). Successive seam cell identities are partially assigned according to the given hypergraphical model in a pair-wise fashion. Each branch greedily queues hypothesized point-pair assignments conditioned on the previous branch assignments (black arrows within a branch). The next leading \textit{V6} pair (Fig~\ref{fig:tree}-E) is chosen upon exhaustion of the leading hypothesized \textit{V6} pair (Fig~\ref{fig:tree}-B). \textit{EHGM} continues the recursion to implicitly identify a globally optimal posture under the given hypergraphical model; each possible initial pair will follow this illustrated process subject to pruning of the minimizing posture accessed via the hypothesized tail pair in Fig~\ref{fig:tree}-A. 

\begin{figure}[!ht]
\centering
\includegraphics[width=\textwidth]{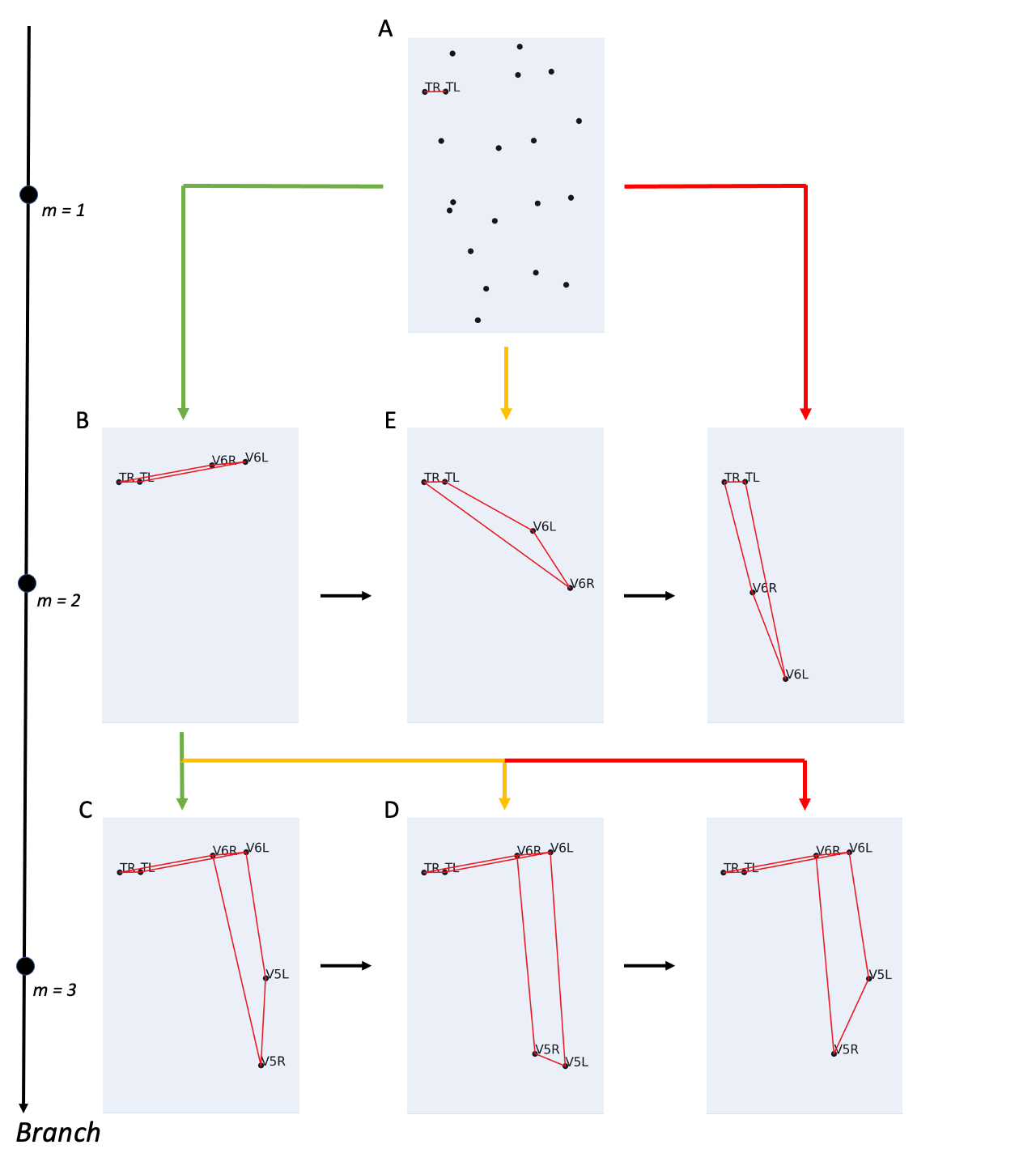}
\caption{\textbf{\textit{EHGM} applied to the sample image depicted in Fig~\ref{fig:twist_straight_3d}-A.} A: Two points are selected at the initial branch for \textit{TL} and \textit{TR}, respectively. Candidates for the successive pair, \textit{V6L} and \textit{V6R}, are queued based on hypergraphical relationships between the established cell identities \textit{TL} and \textit{TR} and each hypothesized \textit{V6} pair (lower costs are green to higher costs in red). B: The leading hypothesis at branch \textit{m}=2 given the initial branch pair is chosen. The recursion continues to queue \textit{V5} pair choices at branch \textit{m}=3. Black arrows within branch \textit{m} specify the ordering of the branch given established cell assignments. Each branch creates a new subproblem of completing the posture given partially assigned identities. C: The tree continuing from the \textit{V5} pair hypothesis is fully explored according to the established recursion. D: The next leading \textit{V5} hypothesis is initiated upon exhaustion of the subtree formed at panel C. E: Implicit enumeration of the subtree formed at panel B causes the search to progress to the second leading \textit{V6} hypothesis.}
\label{fig:tree}
\end{figure}

\section*{Acknowledgments}

This work utilized the computational resources of the NIH HPC Biowulf cluster. (http://hpc.nih.gov). Dr. Evan Ardiel was instrumental in developing descriptive features for identifying worm posture. Post-Baccelaureate research fellows Brandon Harvey and Nensi Karaj were supportive in providing data and discussions concerning the modeling. Dr. Zhen Zhang and Dr. Arye Nehorai provided assistance in using \textit{KerGM} \cite{zhang_kergm_2019}. Dr. Vincent Lyzinski also provided insight on the methods. We also thank Dr. Hank Eden and Dr. Matthew Guay for their careful readings and suggestions. The code and data are available at \url{https://github.com/lauziere/EHGM}.

\clearpage

\section*{Supporting Information}

\setcounter{figure}{0}
\renewcommand{\figurename}{Supp Fig.}

\subsection*{\textit{EHGM} Pseudocode}
\label{SI:pseudocode}

\textit{EHGM} requires as input the branching step $k \in \{1, 2, \dots, n_1\}$, the dissimilarity tensors $\mathbf{Z}^{(d)}$, $d=1, 2, \dots, n_1$, the size $k$ permutation set $\mathbf{P}$, and optionally an initial upper bound $C_0$ on the global minimum $C^*$. 
The dissimilarity tensors are calculated given the reference hypergraph of size $n_1$, and either a previous frame hypergraph or a template hypergraph as described in \textit{Model Fitting}. The lower degree dissimilarity tensors $\mathbf{Z}^{(d)}, d \leq 2k$ are calculated prior to the search and used to select branches. The higher degree dissimilarity terms $d > 2k$ are calculated during the search as required. Algorithm \ref{Exact_HGM} initializes the search from the first candidate set $\mathbf{Q}_1 = \mathbf{P}$. The search is parallelized via initializing several first branches. Each explores a disjoint section of the domain $\mathcal{X}$.

Algorithm \ref{Exact_HGM} initializes arrays and variables to start the recursive branch search (Algorithm \ref{Visit}). Eligible branch candidates are subset from the general queue $\mathbf{P}$ into $\mathbf{Q}_m$ via the \textit{Enqueue} procedure (Algorithm \ref{Enqueue}). Each $\mathbf{Q}_m$ contains the potential assignments for the next \textit{k} terms that satisfy both the pruning constraints and assignment constraints specified by $\mathcal{X}$. The current assignment cost $\Tilde{C}$ is checked against the current minimum $C^*$ upon reaching a complete assignment. The \text{Backtrack} procedure (Algorithm \ref{Backtrack}) removes $\mathbf{k}_{m-1}$ from $\mathbf{Q}_{m-1}$ when the path from $\mathbf{k}_{m-1}$ is exhausted, which occurs when $\mathbf{Q}_m = \emptyset$. The recursion will continue until $\mathbf{Q}_1$ is empty, signaling the complete enumeration of the search space $S_n$. 

\begin{algorithm}[H]
\SetAlgoLined
Input: $k, C_0, \mathbf{P}, \mathbf{Z}^{(1)}, \dots, \mathbf{Z}^{(2k)}$ \\
Output: $\mathbf{x}^*, C^* = \mathit{f}(\mathbf{x}^*)$ \\

Initialization \\ 
$C^* \leftarrow  C_0$ \\
$\Tilde{\mathbf{H}} = []$ \\
$\Tilde{\mathbf{I}} = []$ \\

$\Tilde{\mathbf{x}} \leftarrow \emptyset$ \\
$\Tilde{C} \leftarrow 0$ \\
$m \leftarrow 1$ \\

$\mathbf{Q}_1 \leftarrow Enqueue(\Tilde{\mathbf{x}}, \mathbf{P}, \Tilde{C}, C^*, 1)$ \\
\While{$\mathbf{Q}_1 \neq \emptyset$}{
    
    $\mathbf{k}_1 \leftarrow \mathbf{Q}_1.pop()$ \\
    $\hat{H_1} \leftarrow H_1(\mathbf{k}_1 | \mathbf{Z}^{(1)}, \dots, \mathbf{Z}^{(k)})$ \\
    $\Tilde{C} \leftarrow \hat{H_1}$ \\
    $\Tilde{\mathbf{H}}[1] = \hat{H_1}$ \\
    $\Tilde{\mathbf{x}} \leftarrow \Tilde{\mathbf{x}} \cup \{ \mathbf{k}_1 \}$ \\
    $m \mathrel{+}= 1$ \\
    $Visit(\mathbf{P},  \Tilde{\mathbf{x}}, \Tilde{C}, C^*, m)$ \\
        
}

Return: $\mathbf{x}^*, C^*$ \\
\caption{EHGM}
\label{Exact_HGM}
\end{algorithm}

\begin{algorithm}[H]
\SetAlgoLined
Input: $\mathbf{P}, \Tilde{\mathbf{x}}, \Tilde{C}, C^*, m$ \\

$\mathbf{Q}_m \leftarrow Enqueue(\Tilde{\mathbf{x}}, \mathbf{P}, \Tilde{C}, C^*, m)$ \\
 
\While{$\mathbf{Q}_m \neq \emptyset$}{
    
    $\hat{\mathbf{k}}_{m} \leftarrow \mathbf{Q}_m.pop()$ \\
    $\Tilde{\mathbf{x}} \leftarrow \Tilde{\mathbf{x}} \cup \{ \hat{\mathbf{k}}_{m} \}$ \\
    $m \mathrel{+}= 1$ \\
    $\hat{H}_m \leftarrow H_m(\mathbf{k}_m | \Tilde{\mathbf{x}}, \mathbf{Z}^{(1)}, \dots, \mathbf{Z}^{(2k)})$ \\
    $\Tilde{\mathbf{H}}[m] = \hat{H}_m$ \\
    $\Tilde{C} \mathrel{+}= \hat{H}_m$ \\
    
    \If{$m \geq 3$}{
        
        $\hat{I}_m \leftarrow I_m(\mathbf{k}_m | \Tilde{\mathbf{x}}, \mathbf{Z}^{(2k+1)}, \dots, \mathbf{Z}^{(mk)})$ \\
        $\Tilde{\mathbf{I}}[m] = \hat{I}_m$ \\
        
        $\Tilde{C} \mathrel{+}= \hat{I}_m$ \\
    
    }
    
    \uIf{$m < M$}{
         $Visit(\mathbf{P},  \Tilde{\mathbf{x}}, \Tilde{C}, C^*, m)$ \\
    
    }
    
    \uElseIf{$m = M$}{

        \If{$\Tilde{C} \leq C^*$}{
            $\mathbf{x}^* \leftarrow \Tilde{\mathbf{x}}$ \\
            $C^* \leftarrow \Tilde{C}$ \\
        }
    }
    
    $Backtrack(\Tilde{\mathbf{x}}, \mathbf{Q}_{m-1}, m)$ \\
    
}
 
\caption{Visit}
\label{Visit}
\end{algorithm}
 
\clearpage

\begin{algorithm}[H]
\SetAlgoLined
Input: $\Tilde{\mathbf{x}}, \mathbf{P}, \Tilde{C}, C^*, m$ \\
Output: $\mathbf{Q}_m$ \\

$\mathbf{Q}_m \leftarrow \emptyset$ \\

\For{$\mathbf{k} \in  \mathbf{P}$}{
    
    \If{$(\mathbf{k} \cap \Tilde{\mathbf{x}} = \emptyset) \wedge (\Tilde{C} + H_m(\mathbf{k} | \Tilde{\mathbf{x}}) < C^*)$}{
        
        $ \mathbf{Q}_{m} \leftarrow \mathbf{Q}_{m} \cup \mathbf{k}$
    }
 }

\caption{Enqueue}
\label{Enqueue}
\end{algorithm}

\begin{algorithm}[H]
\SetAlgoLined
Input: $\Tilde{\mathbf{x}}, \mathbf{Q}_{m-1}, m$ \\

$\mathbf{Q}_{m-1} \leftarrow \mathbf{Q}_{m-1} \setminus \Tilde{\mathbf{x}}_m$ \\
$\Tilde{\mathbf{x}} \leftarrow \{\hat{\mathbf{x}}_1, \hat{\mathbf{x}}_2, \dots, \hat{\mathbf{x}}_{m-1}\}$ \\
$\Tilde{C} \leftarrow \Sigma_{j=1}^{m-1} (\mathbf{C}_m + \mathbf{I}_m)$

\caption{Backtrack}
\label{Backtrack}
\end{algorithm}

\subsection*{Hypergraphical Objective Decomposition}
\label{SI:hyp_obj}
{\bf The hypergraphical optimization objective can be decomposed according to hyperedge multiplicity and branching step. The stratification enables efficient search via \textit{EHGM}.} 

\newtheorem{theorem}{Theorem}

\begin{theorem}

Assume an assignment problem objective \textit{f} is in the form:

\begin{multline}
\label{eqn:additive}
    \mathit{f}(X | \mathbf{Z}^{(1)}, \mathbf{Z}^{(2)}, \dots,  \mathbf{Z}^{(n_1)}) = \sum_{l_1=1}^{n_1} \sum_{l'_1=1}^{n_2} \mathbf{Z}^{(1)}_{l_1 l'_1} x_{l_1 l'_1} + \sum_{l_1=1}^{n_1} \sum_{l'_1=1}^{n_2} \sum_{l_2=l_1+1}^{n_1} \sum_{l'_2=1}^{n_2} \mathbf{Z}^{(2)}_{l_1 l'_1 l_2 l'_2} x_{l_1 l'_1} x_{l_2 l'_2}  \\ + \sum_{l_1=1}^{n_1} \sum_{l'_1=1}^{n_2} \sum_{l_2=l_1+1}^{n_1} \sum_{l'_2=1}^{n_2} \sum_{l_3=l_2+1}^{n_1} \sum_{l'_3=1}^{n_2} \mathbf{Z}^{(3)}_{l_1 l'_1 l_2 l'_2 l_3 l'_3} x_{l_1 l'_1} x_{l_2 l'_2} x_{l_3 l'_3}  + ... \\ + \sum_{l_1=1}^{n_1} \sum_{l'_1=1}^{n_2}  ... \sum_{l_{n_1}=l_{n_1-1}+1}^{n_1} \sum_{l'_{n_1}=1}^{n_2} \mathbf{Z}^{(n_1)}_{l_1 l'_1 \dots l_{n_1} l'_{n_1}} x_{l_1 l'_1} \dots x_{l_{n_1} l'_{n_1}}
\end{multline}

Then, for $k \in \{1, 2, \dots, n_1$\}, the stratification fully describes the objective \textit{f} after $M=\frac{n_1}{k}$ branches. Define of $H_1$, $H_m$, and $I_m$: 

\begin{multline*}
    H_{1}(\mathbf{K}_1 | \mathbf{Z}^{(1)}, \mathbf{Z}^{(2)}, ..., \mathbf{Z}^{(k)}) \vcentcolon = \\ \sum_{i_1=1}^{k} \mathbf{Z}^{(1)}_{l_{i_1} l'_{i_1}} + \sum_{i_1=1}^{k} \sum_{i_2=i_1+1}^{k} \mathbf{Z}^{(2)}_{l_{i_1} l'_{i_1} l_{i_2} l'_{i_2}} + ... + \sum_{i_1=1}^{k} \sum_{i_2=i_1+1}^{k} \dots \sum_{i_k=i_{k-1}+1}^{k} \mathbf{Z}^{(k)}_{l_1 l'_{i_1} l_{i_2} l'_{i_2} \dots l_{i_k} l'_{i_k}}
\end{multline*}

\begin{multline*}
    H_{m}(\mathbf{K}_{m} | \mathbf{K}_{1}, ..., \mathbf{K}_{m-1}, \mathbf{Z}^{(1)}, ..., \mathbf{Z}^{(2k)}) \vcentcolon = \\
    \sum_{i_1=(m-1)k+1}^{mk} \mathbf{Z}^{(1)}_{l_{i_1} l'_{i_1}} + \sum_{i_2=(m-1)k+1}^{mk} \sum_{i_1=1}^{i_2-1} \mathbf{Z}^{(2)}_{l_{i_1} l'_{i_1} l_{i_2} l'_{i_2}} \\ + \sum_{i_3=(m-1)k+1}^{mk} \sum_{i_2=1}^{i_3-1} \sum_{i_1=1}^{i_2-1} \mathbf{Z}^{(3)}_{l_{i_1} l'_{i_1} l_{i_2} l'_{i_2} l_{i_3} l'_{i_3}} + ... +  \sum_{i_{2k} = (m-1)k+1}^{mk} \sum_{i_{2k-1} = 1}^{i_{2k}}  \dots \sum_{i_1 = 1}^{i_2-1} \mathbf{Z}^{(2k)}_{l_{i_1} l'_{i_1} \dots l_{i_{2k}} l'_{i_{2k}}}
\end{multline*}

\begin{equation*}
    I_m(\mathbf{K}_{m} | \mathbf{K}_{1}, \mathbf{K}_{2}, \dots, \mathbf{K}_{m-1}, \mathbf{Z}^{(2k+1)}, \dots, \mathbf{Z}^{(mk)}) \vcentcolon = \sum_{d=2k+1}^{mk} \Xi^{(d)}_m
\end{equation*}

where

\begin{multline*}
    \Xi^{(d)}_m(\mathbf{K}_{m} | \mathbf{K}_{1}, \mathbf{K}_{2}, \dots, \mathbf{K}_{m-1}, \mathbf{Z}^{(2k+1)}, \dots, \mathbf{Z}^{(mk)}) \vcentcolon =  \\ \sum_{i_d=(m-1)k+1}^{mk} \sum_{i_{d-1}=1}^{i_d-1} ... \sum_{i_1=1}^{i_2-1} \mathbf{Z}^{(d)}_{l_{i_1} l'_{i_1} ... l_{i_d} l'_{i_d}} 
\end{multline*}

Then, the degree $n_1$ hypergraph matching objective \textit{f} can be expressed 

\begin{equation*}
    \mathit{f}(X | \mathbf{Z}^{(1)}, \mathbf{Z}^{(2)}, \dots,  \mathbf{Z}^{(n_1)}) = \sum_{m=1}^m \mathit{H}_m + \sum_{m=3}^m I_m
\end{equation*}

\end{theorem}

\begin{proof}

    First consider the single branching case $k=1$. This yields $M=\frac{n_1}{k}=\frac{n_1}{1}=n_1$ branches. Each branch yields one assignment; i.e. $K_m = l'_m$ is assigned to the vertex $l_m$. The initial branch selection rule $H_1$ can only utilize the first order term:
    
    \begin{equation*}
        H_1(K_1 | \mathbf{Z}^{(1)}) = \mathbf{Z}^{(1)}_{l_1 l'_1}
    \end{equation*}
    
    Then the general selection rule for the second branch will: gather the first order costs for the second assignment as well as the quadratic (second order) costs between the first two assignments: 
    
    \begin{equation*}
        H_2(K_2 | K_1, \mathbf{Z}^{(1)}, \mathbf{Z}^{(2)}) = \mathbf{Z}^{(1)}_{l_2 l'_2} + \sum_{i_2=2}^{2} \sum_{i_1=1}^{i_2} \mathbf{Z}^{(2)}_{l_{i_1} l'_{i_1} l_{i_2} l'_{i_2}} = \mathbf{Z}^{(1)}_{l_2} + \mathbf{Z}^{(2)}_{l_1 l'_1 l_2 l'_2}
    \end{equation*}
    
    The third branching step will include $H_3$ and $I_3$. $H_3$ follows from $H_2$:
    
    \begin{equation*}
        H_3(K_3 | K_1, K_2, \mathbf{Z}^{(1)}, \mathbf{Z}^{(2)}) = \mathbf{Z}^{(1)}_{l_3 l'_3} + \sum_{i_2=3}^{3} \sum_{i_1=1}^{i_2-1} \mathbf{Z}^{(2)}_{l_{i_1} l'_{i_1} l_{i_2} l'_{i_2}} = \mathbf{Z}^{(1)}_{l_3 l'_3} + \mathbf{Z}^{(2)}_{l_1 l'_1 l_3 l'_3} +  \mathbf{Z}^{(2)}_{l_2 l'_2 l_3 l'_3}
    \end{equation*}

    \begin{equation*}
        I_3(K_3 |  K_1, K_2, \mathbf{Z}^{(3)}) = \Xi^{(3)}_m = \mathbf{Z}^{(3)}_{l_1 l'_1 l_2 l'_2 l_3 l'_3}
    \end{equation*}
    
    Note that if $n_1=3$, then $H_1 + H_2 + H_3 + I_3$ fully describes the third order assignment problem:
    
    \begin{multline}
        H_1 + H_2 + H_3 + I_3 \\ =
         \underbrace{\mathbf{Z}^{(1)}_{l_1 l'_1}}_{H_1} + \underbrace{\mathbf{Z}^{(1)}_{l_2 l'_2} + \mathbf{Z}^{(2)}_{l_1 l'_1 l_2 l'_2}}_{H_2} + \underbrace{\mathbf{Z}^{(1)}_{l_3 l'_3} + \mathbf{Z}^{(2)}_{l_1 l'_1 l_3 l'_3} + \mathbf{Z}^{(2)}_{l_2 l'_2 l_3 l'_3}}_{H_3} + \underbrace{\mathbf{Z}^{(3)}_{l_1 l'_1 l_2 l'_2 l_3 l'_3}}_{I_3} \\
         = \sum_{i_1=1}^{3} \mathbf{Z}^{(1)}_{l_{i_1} l'_{i_1}} + \sum_{i_1=1}^{3} \sum_{i_2=i_1 + 1}^{3} \mathbf{Z}^{(2)}_{l_{i_1} l'_{i_1} l_{i_2} l'_{i_2}} + \sum_{i_1=1}^{3} \sum_{i_2=i_1+1}^{3} \sum_{i_3 = i_2+1}^{3}  \mathbf{Z}^{(3)}_{l_{i_1} l'_{i_1} l_{i_2} l'_{i_2} l_{i_3} l'_{i_3}} \\ = f(X | \mathbf{Z}^{(1)}, \mathbf{Z}^{(2)}, \mathbf{Z}^{(3)})
    \end{multline}
    
    Now consider the extension to $n_1=4$, yielding a fourth degree assignment problem. The fourth branch will assign the next term, $K_4 = l'_4$. The terms $H_4$ and $I_4$ will then fully specify the fourth degree problem:
    
    \begin{equation*}
        H_4(K_4 | K_1, K_2, K_3, \mathbf{Z}^{(1)}, \mathbf{Z}^{(2)}) = \mathbf{Z}^{(1)}_{l_4 l'_4} + \mathbf{Z}^{(2)}_{l_1 l'_1 l_4 l'_4} + \mathbf{Z}^{(2)}_{l_2 l'_2 l_4 l'_4} + \mathbf{Z}^{(2)}_{l_3 l'_3 l_4 l'_4} 
    \end{equation*}
     
    The second aggregation rule $I_4$ will consider third order terms between branches $1, 2$ and $4$ as well as the fourth order term using all four assignments:
    
    \begin{equation*}
        I_4(K_4 | K_1, K_2, K_3, \mathbf{Z}^{(3)}, \mathbf{Z}^{(4)}) = \Xi^{(3)}_{4} + \Xi^{(4)}_{4} = \mathbf{Z}^{(3)}_{l_1 l'_1 l_2 l'_2 l_4 l'_4} + \mathbf{Z}^{(3)}_{l_2 l'_2 l_3 l'_3 l_4 l'_4} + \mathbf{Z}^{(4)}_{l_1 l'_1 l_2 l'_2 l_3 l'_3 l_4 l'_4} 
    \end{equation*}
    
    Joining the fourth branch: 
    
    \begin{multline}
        H_1 + H_2 + H_3 + I_3 + H_4 + I_4 = \\
         \underbrace{\sum_{i_1=1}^{3} \mathbf{Z}^{(1)}_{l_{i_1} l'_{i_1}} + \sum_{i_1=1}^{3} \sum_{i_2=i_1+1}^{3} \mathbf{Z}^{(2)}_{l_{i_1} l'_{i_1} l_{i_2} l'_{i_2}} + \sum_{i_1=1}^{3} \sum_{i_2=i_1+1}^{3} \sum_{i_3=i_2+1}^{3} \mathbf{Z}^{(3)}_{l_{i_1} l'_{i_1} l_{i_2} l'_{i_2} l_{i_3} l'_{i_3}}}_{H_1 + H_2 + H_3 + I_3} \\ +  \underbrace{\mathbf{Z}^{(1)}_{l_4 l'_4} + \mathbf{Z}^{(2)}_{l_1 l'_1 l_4 l'_4} + \mathbf{Z}^{(2)}_{l_2 l'_2 l_4 l'_4} + \mathbf{Z}^{(2)}_{l_3 l'_3 l_4 l'_4}}_{H_4} + \underbrace{\mathbf{Z}^{(3)}_{l_1 l'_1 l_2 l'_2 l_4 l'_4} + \mathbf{Z}^{(3)}_{l_2 l'_2 l_3 l'_3 l_4 l'_4} + \mathbf{Z}^{(4)}_{l_1 l'_1 l_2 l'_2 l_3 l'_3 l_4 l'_4}}_{I_4} \\ = \sum_{i_1=1}^{4} \mathbf{Z}^{(1)}_{l_{i_1} l'_{i_1}} + \sum_{i_1=1}^{4} \sum_{i_2=i_1+1}^{4} \mathbf{Z}^{(2)}_{l_{i_1} l'_{i_1} l_{i_2} l'_{i_2}} + \sum_{i_1=1}^{4} \sum_{i_2=i_1+1}^{4} \sum_{i_3=i_2+1}^{4} \mathbf{Z}^{(3)}_{l_{i_1} l'_{i_1} l_{i_2} l'_{i_2} l_{i_3} l'_{i_3}} \\ + \sum_{i_1=1}^{4} \sum_{i_2=i_1+1}^{4} \sum_{i_3=i_2+1}^{4} \sum_{i_4=i_3+1}^{4} \mathbf{Z}^{(4)}_{l_{i_1} l'_{i_1} l_{i_2} l'_{i_2} l_{i_3} l'_{i_3} l_{i_4} l'_{i_4}} \\ = f(X | \mathbf{Z}^{(1)}, \mathbf{Z}^{(2)}, \mathbf{Z}^{(3)}, \mathbf{Z}^{(4)})
    \end{multline}
    
    Now consider the arbitrary $(m+1)^{st}$ branch. This will yield the full objective for an assignment problem of size $m+1$ up to degree $m+1$. 
    
    \begin{multline}
        \sum_{p=1}^{m+1} H_p + \sum_{p=3}^{m+1} I_p = \sum_{p=1}^{m} H_p + \sum_{p=3}^{m} I_p + H_{m+1} + I_{m+1} = \\ \underbrace{\sum_{i_1=1}^{m} \mathbf{Z}^{(1)}_{l_{i_1} l'_{i_1}} + \sum_{i_1=1}^{m} \sum_{i_2=i_1+1}^{m} \mathbf{Z}^{(2)}_{l_{i_1} l'_{i_1} l_{i_2} l'_{i_2}} + \dots + \sum_{i_1=1}^{m} \sum_{i_2=i_1+1}^{m} ... \sum_{i_m=i_{m-1}+1}^{m} \mathbf{Z}^{(m)}_{l_{i_1} l'_{i_1} l_{i_2} l'_{i_2} ... l_{i_m} l'_{i_m}}}_{\sum_{p=1}^{m} H_p + \sum_{p=3}^{m} I_p} \\ + \underbrace{\mathbf{Z}^{(1)}_{l_{i_{m+1}} l'_{i_{m+1}}} + \sum_{i_1=1}^{m} \mathbf{Z}^{(2)}_{l_{i_1} l'_{i_1} l_{m+1} l'_{m+1}}}_{H_{m+1}} + \underbrace{\sum_{d=3}^{m+1} \Xi^{(d)}_{m+1}}_{I_{m+1}}
    \end{multline}
    
   It is sufficient to show each degree $d \in \{1, 2, \dots, m+1\}$ hyperedge is fully accounted for across all $m+1$ points to prove the $(m+1)^{st}$ branch satisfies the objective \textit{f}. The hyperedge costs across all points will be decomposed into three disjoint sets, and each set considered at a time:
    
    \begin{equation*}
        \{1\}, \{2\}, \{3, \dots, m\}, \{m+1\}
    \end{equation*}
    
    The first and final of the four cases are trivial. The first degree terms are enumerated via the first term in $H_{m+1}$, while $Xi^{(m+1)}_{m+1}$ explicitly addresses the degree $m+1$ hyperedge comprising all assignments: $\mathbf{Z}^{(m+1)}_{l_1 l'_1 l_2 l'_2 ... l_{m+1} l'_{m+1}}$. We will focus on the second and third cases. The degree $d=2$ terms are formed by the addition of branch $m+1$ are considered in term $H_{m+1}$:
    
    \begin{equation*}
        \sum_{i_1=1}^{m} \sum_{i_2=i_1+1}^{m} \mathbf{Z}^{(2)}_{l_{i_1} l'_{i_1} l_{i_2} l'_{i_2}} + \sum_{i_1=1}^{m} \mathbf{Z}^{(2)}_{l_{i_1} l'_{i_1} l_{i_{m+1}} l'_{i_{m+1}}}  = \sum_{i_1 = 1}^{m+1} \sum_{i_2 = i_1+1}^{m+1} \mathbf{Z}^{(2)}_{l_{i_1} l'_{i_1} l_{i_2} l'_{i_2}}
    \end{equation*}
    
     Let $d \in \{3, \dots, m\}$. The completion is similar to the $d=2$ degree case; however, the term $\Xi^{(d)}_{m+1}$ in $I_{m+1}$ address higher degree hyperedges up to and including degree $m$ concerning branch $m+1$:
     
     \begin{equation*}
        \sum_{i_1=1}^{m} ... \sum_{i_d=i_{d-1}+1}^{m} \mathbf{Z}^{(d)}_{l_{i_1} l'_{i_1} ... l_{i_d} l'_{i_d}} + \Xi^{(d)}_{m+1} = \sum_{i_1=1}^{m+1} ... \sum_{i_d=i_{d-1}+1}^{m+1} \mathbf{Z}^{(d)}_{l_{i_1} l'_{i_1} ... l_{i_d} l'_{i_d}}
    \end{equation*}
    
    Therefore, the $(m+1)^{st}$ step fully accrues the objective \textit{f}:
    
    \begin{equation*}
        \sum_{p=1}^{m+1} H_p + \sum_{p=3}^{m+1} I_p = f(X | \mathbf{Z}^{(1)}, \mathbf{Z}^{(2)}, \dots, \mathbf{Z}^{(m+1)})
    \end{equation*}
    
    Then inductively, the stratification holds such that: 
    
    \begin{equation*}
        \sum_{m=1}^{n_1} H_m + \sum_{m=3}^{n_1} I_m = f(X | \mathbf{Z}^{(1)}, \mathbf{Z}^{(2)}, \dots \mathbf{Z}^{(n_1)})
    \end{equation*}
    
    Now consider the plural branching rule $k>1$. The proof will follow from the single assignment branching case. The base case at the fourth branch will be established, followed by the induction hypothesis demonstrating the branching from $m$ to $m+1$. First, define the terms $H_1, H_2, H_3, I_3, H_4$, and $I_4$:
    
    \begin{equation}
        H_1 = \sum_{i_1=1}^{k} \mathbf{Z}^{(1)}_{l_{i_1} l'_{i_1}} + \sum_{i_1=1}^{k} \sum_{i_2=i_1+1}^{k} \mathbf{Z}^{(2)}_{l_{i_1} l'_{i_1} l_{i_2} l'_{i_2}} + ... + \sum_{i_1=1}^{k} \sum_{i_2=i_1+1}^{k} ... \sum_{i_k=i_{k-1}+1}^{k} \mathbf{Z}^{(k)}_{l_{i_1} l'_{i_1} l_{i_2} l'_{i_2} ... l_{i_k} l'_{i_k}}
    \end{equation}
    
    \begin{equation}
        H_2 = \sum_{i_1=k+1}^{2k} \mathbf{Z}^{(1)}_{l_{i_1} l'_{i_1}} + \sum_{i_2=k+1}^{2k} \sum_{i_1=1}^{i_2-1} \mathbf{Z}^{(2)}_{l_{i_1} l'_{i_1} l_{i_2} l'_{i_2}} + ... + \sum_{i_{2k}=k+1}^{2k} ... \sum_{i_2=1}^{i_3-1} \sum_{i_1=1}^{i_2-1} \mathbf{Z}^{(2k)}_{l_{i_1} l'_{i_1} ... l_{i_{2k}} l'_{i_{2k}}} 
    \end{equation}
    
    \begin{equation}
        H_3 = \sum_{i_1=2k+1}^{3k} \mathbf{Z}^{(1)}_{l_{i_1} l'_{i_1}} + \sum_{i_2=2k+1}^{3k} \sum_{i_1=1}^{i_2-1} \mathbf{Z}^{(2)}_{l_{i_1} l'_{i_1} l_{i_2} l'_{i_2}} + ... + \sum_{i_{2k}=2k+1}^{3k} ... \sum_{i_1=1}^{i_2-1} \mathbf{Z}^{(2k)}_{l_{i_1} l'_{i_1} ... l_{i_{2k}} l'_{i_{2k}}} 
    \end{equation}
    
    \begin{equation}
        I_3 = \sum_{d=2k+1}^{3k} \Xi^{(d)}_{3} 
    \end{equation}
    
    \begin{equation}
        H_4 = \sum_{i_1=3k+1}^{4k} \mathbf{Z}^{(1)}_{l_{i_1} l'_{i_1}} + \sum_{i_2=3k+1}^{4k} \sum_{i_1=1}^{i_2-1} \mathbf{Z}^{(2)}_{l_{i_1} l'_{i_1} l_{i_2} l'_{i_2}} + ... + \sum_{i_{2k}=3k+1}^{4k} ... \sum_{i_1=1}^{i_2-1} \mathbf{Z}^{(2k)}_{l_{i_1} l'_{i_1} ... l_{i_{2k}} l'_{i_{2k}}} 
    \end{equation}
    
    \begin{equation}
        I_4 = \sum_{d=3k+1}^{4k} \Xi^{(d)}_{4} 
    \end{equation}
    
    The terms presented thus far for the general $k>1$ case fully describe all terms concerning assignments $1, 2, \dots 4k$ up to degree $4k$. The hyperedge multiplicities will again be partitioned into disjoint groups: 
    
    \begin{equation*}
        \{1\}, \{2, \dots, k\}, \{k+1, \dots, 2k\}, \{2k+1, \dots, 3k\}, \{3k+1, \dots, 4k\}
    \end{equation*}
    
    The first case is trivial, just as in the single assignment branching ($k=1$) proof. Unary terms are accounted for in the first summand of each $H_m$. Then, consider $d \in \{2, \dots, k\}$:
    
    \begin{multline*}
        \underbrace{\sum_{i_1=1}^{k} ... \sum_{i_d=i_{d-1}+1}^{k}  \mathbf{Z}^{(d)}_{l_{i_1} l'_{i_1} ... l_{i_d} l'_{i_d}}}_{H_1} + \underbrace{\sum_{i_d=k+1}^{2k} \sum_{i_{d-1}=1}^{i_d-1} ... \sum_{i_1=1}^{i_2-1}  \mathbf{Z}^{(d)}_{l_{i_1} l'_{i_1} ... l_{i_d} l'_{i_d}}}_{H_2} \\ + \underbrace{\sum_{i_d=2k+1}^{3k} \sum_{i_{d-1}=1}^{i_d-1} ... \sum_{i_1=1}^{i_2-1}  \mathbf{Z}^{(d)}_{l_{i_1} l'_{i_1} ... l_{i_d} l'_{i_d}}}_{H_3} + \underbrace{\sum_{i_d=3k+1}^{4k} \sum_{i_{d-1}=1}^{i_d-1} ... \sum_{i_1=1}^{i_2-1}  \mathbf{Z}^{(d)}_{l_{i_1} l'_{i_1} ... l_{i_d} l'_{i_d}}}_{H_4}\\ = \sum_{i_1=1}^{4k} \sum_{i_2=i_1+1}^{4k} ... \sum_{i_d=i_{d-1}+1}^{4k} \mathbf{Z}^{(d)}_{l_{i_1} l'_{i_1} ... l_{i_d} l'_{i_d}}
    \end{multline*}
    
    The proof for degree $d \in \{k+1, \dots, 2k\}$ follows immediately from the grouping presented above, but without the initial branch selection rule term $H_1$. Next, assume $d \in \{2k+1, \dots, (m-1)k\}$. Degree $d$ hyperedge dissimilarities will be contained in both $I_3$ and $I_4$ terms: 
    
    \begin{multline*}
        \underbrace{\sum_{i_d = 2k+1}^{3k} \sum_{i_{d-1} = 1}^{i_d-1} ... \sum_{i_1 = 1}^{i_2-1} \mathbf{Z}^{(d)}_{l_{i_1} l'_{i_1} ... l_{i_{d}} l'_{i_{d}}}}_{I_3} +  \underbrace{\sum_{i_d = 3k+1}^{4k} \sum_{i_{d-1} = 1}^{i_d-1} ... \sum_{i_1 = 1}^{i_2-1} \mathbf{Z}^{(d)}_{l_{i_1} l'_{i_1} ... l_{i_{d}} l'_{i_{d}}}}_{I_4} \\ = \sum_{i_1 = 1}^{4k} ... \sum_{i_{d-1} = i_{d-2}+1}^{4k} \sum_{i_d = i_{d-1}+1}^{4k} \mathbf{Z}^{(d)}_{l_{i_1} l'_{i_1} ... l_{i_{d}} l'_{i_{d}}}
    \end{multline*}
    
    Since $d \leq 2k+1$, the terms only appear in the third branch term $I_3$ when the assignment $2k+1$ is committed. The final set arises from the definition of $I_4$ which accrues hyperedges of degree $d \in \{3k+1, \dots, 4k\}$ across assignments in branches $m=1, 2, 3, 4$. The base case is fully established for the arbitrary $k>1$ case. The final step of the proof is to establish the extension of the $(m+1)^{st}$ branch:
    
     \begin{multline}
        \sum_{p=1}^{m+1} H_p + \sum_{p=3}^{m+1} I_p = \sum_{p=1}^{m} H_p + \sum_{p=3}^{m} I_p + H_{m+1} + I_{m+1} = \\ \underbrace{\sum_{i_1=1}^{mk} \mathbf{Z}^{(1)}_{l_{i_1} l'_{i_1}} + \sum_{i_1=1}^{mk} \sum_{i_2=i_1+1}^{mk} \mathbf{Z}^{(2)}_{l_{i_1} l'_{i_1} l_{i_2} l'_{i_2}} + \dots + \sum_{i_1=1}^{mk} \sum_{i_2=i_1+1}^{mk} ... \sum_{i_{mk}=i_{mk-1}+1}^{mk} \mathbf{Z}^{(mk)}_{l_{i_1} l'_{i_1} ... l_{i_{mk}} l'_{i_{mk}}}}_{\sum_{p=1}^{m} H_p + \sum_{p=3}^{m} I_p} \\ + \underbrace{\sum_{i_1=mk+1}^{(m+1)k} \mathbf{Z}^{(1)}_{l_{i_1} l'_{i_1}} + ... + \sum_{i_{2k}=mk+1}^{(m+1)k} \sum_{i_{2k-1}=1}^{i_{2k}-1} ... \sum_{i_1=1}^{i_2-1} \mathbf{Z}^{(2k)}_{l_{i_1} l'_{i_1} ... l_{i_{2k}} l'_{i_{2k}}}}_{H_{m+1}} + \underbrace{\sum_{d=2k+1}^{(m+1)k} \Xi^{(d)}_{m+1}}_{I_{m+1}}
    \end{multline}
    
    The $(m+1)k$ hyperedge multiplicities will be stratified into four groups:
    
    \begin{equation*}
        \{1\}, \{2, \dots, 2k\}, \{2k+1, \dots, mk\}, \{mk+1, \dots, (m+1)k\}    
    \end{equation*} 
    
    Just as in the singular $k=1$ case, the proof for the first and last groups are trivial. The unary terms are again evident from the first term in $H_{m+1}$, while the $mk+1 \leq d \leq (m+1)k$ terms in $I_{m+1}$ fully encapsulates the fourth group. The steps in the remaining two cases will follow that of the $k=1$ case. 
    
    First, assume $d \in \{2, \dots, 2k\}$. The extension of the $(m+1)^{st}$ branch uses exclusively the selection rule $H_{m+1}$: 
    
    \begin{multline}
        \underbrace{\sum_{i_1=1}^{mk} ... \sum_{i_{d-1}=i_{d-2}+1}^{mk} \sum_{i_d=i_{d-1}+1}^{mk}  \mathbf{Z}^{(d)}_{l_{i_1} l'_{i_1} ... l_{i_d} l'_{i_d}}}_{\sum_{p=1}^m H_m} + \underbrace{\sum_{i_d=mk+1}^{(m+1)k} \sum_{i_{d-1}=1}^{i_d-1} ... \sum_{i_1=1}^{i_2-1} \mathbf{Z}^{(d)}_{l_{i_1} l'_{i_1} ... l_{i_d} l'_{i_d}}}_{H_{m+1}} = \\ \sum_{i_1=1}^{(m+1)k} ... \sum_{i_{d-1}=i_{d-2}+1}^{(m+1)k} \sum_{i_d=i_{d-1}+1}^{(m+1)k}  \mathbf{Z}^{(d)}_{l_{i_1} l'_{i_1} ... l_{i_d} l'_{i_d}}
    \end{multline}
    
    Next, assume $d \in \{2k+1, \dots, mk\}$. These terms are captured in $I_{m+1}$ using each definition of $\Xi^{(d)}_{m+1}$:
    
    \begin{multline}
        \sum_{i_1=1}^{mk} ... \sum_{i_{d-1}=i_{d-2}+1}^{mk} \sum_{i_d=i_{d-1}+1}^{mk}  \mathbf{Z}^{(d)}_{l_{i_1} l'_{i_1} ... l_{i_d} l'_{i_d}} + \underbrace{\sum_{i_d=mk+1}^{(m+1)k} \sum_{i_{d-1}=1}^{i_d-1} ... \sum_{i_1=1}^{i_2-1} \mathbf{Z}^{(d)}_{l_{i_1} l'_{i_1} ... l_{i_d} l'_{i_d}}}_{I_{m+1}} = \\ \sum_{i_1=1}^{(m+1)k} ... \sum_{i_{d-1}=i_{d-2}+1}^{(m+1)k} \sum_{i_d=i_{d-1}+1}^{(m+1)k}  \mathbf{Z}^{(d)}_{l_{i_1} l'_{i_1} ... l_{i_d} l'_{i_d}}
    \end{multline}
    
    All four results together show that every degree hyperedge $1, \dots, (m+1)k$ is accounted for in the $(m+1)^{st}$ branch, thus proving the induction hypothesis:
    
    \begin{equation*}
        \sum_{p=1}^{m+1} H_p + \sum_{p=3}^{m+1} I_p = f(X | \mathbf{Z}^{(1)}, \mathbf{Z}^{(2)}, \dots \mathbf{Z}^{((m+1)k)})
    \end{equation*}
    
    The $M^{th}$ branch completes the degree $n_1$ assignment problem of size $n_1$. For any $k \in \{1, 2, \dots n_1\}$, the selection and aggregation rules yield the full degree $n_1$ assignment problem objective:
    
    \begin{equation*}
        \sum_{p=1}^{M} H_p + \sum_{p=3}^{M} I_p = f(X | \mathbf{Z}^{(1)}, \mathbf{Z}^{(2)}, \dots \mathbf{Z}^{(n_1)})
    \end{equation*}
    
\end{proof}

\subsection*{Convergence \& Exactness of \textit{EHGM}} \label{SI:conv_exact}

\newtheorem{theorem2}{Theorem}[section]

\begin{theorem}

\textit{EHGM} (algorithm \ref{Exact_HGM}) will converge to a globally optimal solution of the following hypergraph matching optimization problem given input $k \in \{1, 2, \dots, n_1\}$:

\begin{multline}
    \underset{X \in \mathcal{X}}{\text{min}} \sum_{l_1=1}^{n_1} \sum_{l'_1=1}^{n_2} \mathbf{Z}^{(1)}_{l_1 l'_1} x_{l_1 l'_1} + \sum_{l_1=1}^{n_1} \sum_{l'_1=1}^{n_2} \sum_{l_2=l_1+1}^{n_1} \sum_{l'_2=1}^{n_2} \mathbf{Z}^{(2)}_{l_1 l'_1 l_2 l'_2} x_{l_1 l'_1} x_{l_2 l'_2}  \\ + \sum_{l_1=1}^{n_1} \sum_{l'_1=1}^{n_2} \sum_{l_2=l_1+1}^{n_1} \sum_{l'_2=1}^{n_2} \sum_{l_3=l_2+1}^{n_1} \sum_{l'_3=1}^{n_2} \mathbf{Z}^{(3)}_{l_1 l'_1 l_2 l'_2 l_3 l'_3} x_{l_1 l'_1} x_{l_2 l'_2} x_{l_3 l'_3}  + ... \\ + \sum_{l_1=1}^{n_1} \sum_{l'_1=1}^{n_2}  ... \sum_{l_{n_1}=l_{n_1-1}+1}^{n_1} \sum_{l'_{n_1}=1}^{n_2} \mathbf{Z}^{(n_1)}_{l_1 l'_1 \dots l_{n_1} l'_{n_1}} x_{l_1 l'_1} \dots x_{l_{n_1} l'_{n_1}}
\end{multline}

where $\mathcal{X}$ is defined:

\begin{equation}
    \mathcal{X} = \{X \in \{0,1\}^{n_1 \times n_2}: \forall j, \sum_{i = 1}^{n_1} x_{ij} \leq 1, \forall i \sum_{j = 1}^{n_2} x_{ij} = 1\}
\end{equation}

\end{theorem}

\begin{proof}
First, we will show \textit{EHGM} converges, then it will be proven that the converged solution is globally optimal. 

The search begins with initializing queue $\mathbf{Q}_1 = \mathbf{P}$. The algorithm terminates with the exhaustion of $\mathbf{Q}_1$. Each set $\mathbf{Q}_m \subset \mathbf{P}$ contains feasible \textit{k}-assignments conditioned on the assignment constraints and costs $\Tilde{C}, C^*$. \textit{Backtrack} (Algorithm \ref{Backtrack}) removes $\mathbf{x}_m$ from $\mathbf{Q}_{m-1}$ upon enumeration of $\mathbf{Q}_m$. The recursion then falls back to selecting from branch $m-1$, eventually exhausting $\mathbf{Q}_{m-1}$ just as in the enumeration of $\mathbf{Q}_m$. This recursion continues until the first branch $\mathbf{k}_1 \in \mathbf{Q}_1$ is removed, signaling the exploration of all assignments originating with the \textit{k}-tuple $\mathbf{k}_1$. The exploration is repeated for each $\mathbf{k}_1 \in \mathbf{Q}_1$. Thus, all possible assignments $X \in \mathcal{X}$ are explored via the branching scheme. 

Assignments accrue a monotonically increasing cost $\Tilde{C}$ to be compared to $C^*$ with accompanying assignment $\mathbf{x}^*$ at each branch. A complete assignment then drops the last \textit{k} assignments from $\Tilde{\mathbf{x}}$, initializing the backwards recursion, emptying $\mathbf{Q}_m$ until $\mathbf{k}^{(n_k)}_{m-1} \in \mathbf{Q}_{m-1}$ is exhausted. There are at most $|\mathbf{Q}_m| \leq n_k$ viable permutations at branch \textit{m}. Each possible branch is evaluated from $\Tilde{\mathbf{x}} = [
\mathbf{k}_1, \mathbf{k}_2, \dots \mathbf{k}_{m-1}]$. The $(m-1)^{st}$ branch $\mathbf{k}_{m-1} \in \mathbf{Q}_{m-1}$ is removed from $\mathbf{Q}_{m-1}$ upon exhaustion of $\mathbf{Q}_m$: 

\begin{align*}
    \Tilde{\mathbf{x}}^{(m)}_1 &= [\mathbf{k}^{(1)}_1, \mathbf{k}^{(1)}_2, \dots \mathbf{k}^{(1)}_{m-1}, \mathbf{k}^{(1)}_m] \\ \Tilde{\mathbf{x}}^{(m)}_2 &= [\mathbf{k}^{(1)}_1, \mathbf{k}^{(1)}_2, \dots \mathbf{k}^{(1)}_{m-1}, \mathbf{k}^{(2)}_m] \\ 
    ... \\
    \Tilde{\mathbf{x}}^{(m)}_{n_k} &= [\mathbf{k}^{(1)}_1, \mathbf{k}^{(1)}_2, \dots \mathbf{k}^{(1)}_{m-1}, \mathbf{k}^{(n_k)}_m] \\ 
\end{align*}

Each of the $n_k$ possible final branches from $\mathbf{k}^{(1)}_{m-1}$ is explored, then $\mathbf{k}^{(1)}_{m-1}$ is removed from $\mathbf{Q}_{m-1}$. 



The process follows for the $M^{th}$ branch, exhausting viable assignment sets until $\mathbf{k}_{M-1}$ is removed. The recursion follows inductively back to the exhaustion of $\mathbf{Q}_1$, signaling the end of the search. Thus, all possible assignments $X \in \mathcal{X}$ are explored via the branching scheme. 

The convergent and exhaustive algorithm will yield a globally optimal solution $C^* = f(\mathbf{x}^*)$ after exhausting $\mathbf{Q}_1$. As proven above the additive decomposition of the cost structure (equation \ref{eqn:additive}) is proven to be satisfied by summing all selection and aggregation rule terms. Assume an uninformed initialization $C^* = \infty$. Then the first pass will greedily take the best permutation from the first set $\mathbf{Q}_1$: $\mathbf{k}^{(1)}_1$, and the best from the second set given it does not conflict with $\mathbf{k}^{(1)}_1$: $\mathbf{k}^{(2|1)}_2$. This process will continue until the first complete assignment: $\Tilde{\mathbf{x}} = [\mathbf{k}^{(1)}_1, \mathbf{k}^{(2|1)}_2, \mathbf{k}^{(3|2,1)}_3, \dots, \mathbf{k}^{(M|(M-1), \dots, 1)}_M]$ with $\Tilde{C} = f(\Tilde{\mathbf{x}})$. The first \textit{Bracktrack} removes $\mathbf{k}^{(M|(M-1), \dots, 1)}_M$, and the $M^{th}$ \textit{Visit} call will exhaust $\mathbf{Q}_M$. Subsequent \textit{Enqueue} calls will limit only allow branches that satisfy both the assignment constraints and the updated selection rule cost (Algorithm \ref{Enqueue}). This follows that any \textit{k}-tuple of assignments $\mathbf{k}^{(j)}_m$ such that for $\Tilde{\mathbf{x}} = [\mathbf{k}^{(j_1)}_1, \mathbf{k}^{(j_2|j_1)}_2, \dots, \mathbf{k}^{(j_d|j_{(m-1)}, \dots j_1)}_m]$:

\begin{equation*}
    \Tilde{C} + H_m(\Tilde{\mathbf{x}}_{m-1}, \mathbf{k}^{(j_d|j_{(m-1)}, \dots j_1)}_m) < C^*
\end{equation*}

The additive decomposition of the objective paired with the assumed non-negativity of the dissimilarity tensors $\mathbf{Z}^{(j)}$ results in each branch monotonically increasing $\Tilde{C}$: 

\begin{equation*}
    \Tilde{C} + H_m(\Tilde{\mathbf{x}}_{m-1}, \mathbf{k}^{(j_d|j_{(m-1)}, \dots j_1)}_m) + I_m(\mathbf{k}^{(j_1)}_1, \mathbf{k}^{(j_2|j_1)}_2, \dots, \mathbf{k}^{(j_d|j_{(m-1)}, \dots j_1)}_m) \geq \Tilde{C}
\end{equation*}

The convergent search will thus eliminate all paths that are not globally optimal. Incrementally updating the reserved solution $\mathbf{x}^*$ with cost $C^*$ expedites convergence as each replacement is necessarily a better solution. The resulting $\mathbf{x}^*$ and corresponding cost $C^*$ are such that at no other full assignment $\Tilde{\mathbf{x}}$ can replace $\mathbf{x}^*$, by definition a globally optimal solution of \textit{f}. 

\end{proof}

\subsection*{Model Fitting} \label{SI:DDAPs}

Expert annotations are used to derive features such that the correct assignment consistently achieves a minimal cost across the training set. Features can be engineered and analyzed in context of point set matching just as in traditional supervised learning tasks. 

Features are expressed as attributes over hyperedge multiplicities $d = 1, 2, \dots, n_1$. Hyperedge features $g^{(d)}_s$, $s=1,\dots,n_d$ are given as input. Each feature $g_s^{(d)}$ assumes a Gaussian distribution, and if $n_d \geq 2$ the features are modeled as a multivariate Gaussian distribution. Measurements from the data are used to derive estimates of the parameters of the Gaussian distributions: $\mathbf{mu}$ and $\mathbf{Sigma}$. The most common application in heuristic approaches is to use the previous frame's feature values as the centers of the distributions. This standard approach is effective for features that vary minimally, frame-to-frame. However, certain angle measurements may vary greatly between frames. Mean estimates across the training data can better account for macroscopic patterns in features. The variances are then estimated from the feature values across the training set. 

The dissimilarity costs arise from the Mahalanbonis distance between a hypothesized assignment's feature measurements and the estimated template mean values scaled by estimated covariance matrix. The dissimilarity tensors $\mathbf{Z}^{(d)}$ are expressed as a function of the $n_d$ features of hyperedge \textit{d}. A partial assignment up to degree \textit{d}: $[(l_1, \dots, l_d) \mapsto (l'_1, \dots, l'_d)]$ invokes a cost according to the $n_d$ features: $\sum_{s=1}^{n_d} g_s^{(d)}$. The expected values: $\sum_{s=1}^{n_d} \bar{g}_s^{(d)}$ are calculated in aggregate from training data for higher variance patterns:

\begin{equation} \label{eqn:means}
    \Bar{g}_s^{(d)} = \frac{\sum_{L=1}^N g_s^{(d)}(X_L, \mathbf{X}_L)}{N}
\end{equation}

where $X_L$ and $\mathbf{X}_L$ are the correct permutation and observed point set, respectively, for sample $L$. The variance-covariance matrix uses estimated means to estimate variances and covariances among feature measurements in the annotated data:

\begin{equation}
\label{eqn:covs}
    \hat{\sigma}^{(d)}_{a,b} = \sum_{L=1}^N (g_a^{(d)}(X_L, \mathbf{X}_L) - \Bar{g}_a^{(d)})(g_b^{(d)}(X_L, \mathbf{X}_L) - \Bar{g}_b^{(d)})
\end{equation}

\[
\hat{\mathbf{\Sigma}}^{(d)}_{g} = \left| \begin{bmatrix}
\hat{\sigma}^{(d)}_{1,1}  & \hat{\sigma}^{(d)}_{1,2} & ... & \hat{\sigma}^{(d)}_{1,n_d} \\
\hat{\sigma}^{(d)}_{2,1} & \hat{\sigma}^{(d)}_{2,2} & ... & \hat{\sigma}^{(d)}_{2,n_d} \\
... & ... & ... & ... \\
\hat{\sigma}^{(d)}_{n_d,1} & \hat{\sigma}^{(d)}_{n_d,2} & ... & \hat{\sigma}^{(d)}_{n_d, n_d} \end{bmatrix} \right| 
\]

The selection rule tensor dissimilarity tensors $\mathbf{Z}^{(d)} \in R^{\underbrace{n \times n, \dots, \times n}_{2d}}$ use both sets of estimates to compute costs. The Mahalanobis distance is used to describe the scaled distance between the observed attributed hyperedge to an estimated feature description. Let $\mathbf{g}^{(d)} = [g_1^{(d)}, g_2^{(d)}, \dots, g_{n_d}^{(d)}]'$ and $\bar{\mathbf{g}}^{(d)} = [\bar{g}_1^{(d)}, \bar{g}_2^{(d)}, \dots, \bar{g}_{n_d}^{(d)}]'$ 

\begin{equation}
    \mathbf{Z}^{(d)}_{l_1 l'_1 l_2 l'_2 \dots l_d l'_d} = (\mathbf{g}^{(d)} - \bar{\mathbf{g}}^{(d)})' (\hat{\mathbf{\Sigma}}^{(d)}_{g})^{-1} (\mathbf{g}^{(d)} - \bar{\mathbf{g}}^{(d)})
\end{equation}

The traditional approach uses the labeled coordinates in the prior frame to build corresponding prior frame feature measurements. These prior frame feature measurements then serve as the estimated center of the Gaussian distribution. The covariance matrix estimation follows accordingly, in which the variation in frame-to-frame differences is estimated from sample data. 

\subsection*{Posture Modeling} \label{SI:Modeling}

Embryonic \textit{C. elegans} posture modeling used the aforementioned template hypergraph for quantifying hypothesized seam cell identities throughout the search process. The developing embryo elongates and as a result becomes more coiled due to the constraining eggshell. As such, the template hypergraphs are updated according to binned time intervals. Parameters are estimated from data according to the point in development between first image and hatch. 

Each image is captured at time $t$ with $n_1=20$ located nuclei centroids. The coordinates can be stored as $\mathbf{X} \in R^{n \times 3}$ which $\mathbf{X}_i = [x_i, y_i, z_i]$ representing the $i^{th}$ centroid in $R^3$. The seam cells are ordered posterior to anterior: \textit{TL}, \textit{TR}, \dots, \textit{H0L}, \textit{H0R}. Then let $1, 3, 5, \dots 19$ be in the indices of the left side, and $2, 4, 6, \dots 20$ be the indices of the right side. The seam cells are paired $(1,2)$ for the tail pair, then $(2,3), (4,5), (6,7), \dots (19, 20)$ for the body pairs. Let $\mathbf{L} = (\mathbf{l}_1, \mathbf{l}_2, \dots, \mathbf{l}_{10})$ denote the left side nuclei locations, and similarly $\mathbf{R}$ the right side nuclei locations. 

The sampled worm embryos develop at similar rates. However, the occurrence of the first twitch, a point in development that triggers rapid physical changes, varies slightly embryo-to-embryo. As a result, we apply a time normalization in effort to compare feature measurements from images of the different specimens. Each embryo's time to first twitch is measured $s_{w}$, as well as hatch time $h_w$ for sample embryo $w = 1, 2, \dots 16$. The time points for each sample are indexed $k = 1, 2, \dots n_w$. Each volume's imaging time $t_{wk}$ is normalized to the $[0,1]$ range via Eq~\ref{eqn:time_norm}. Each normalized time point $z_{wk}$ is scaled such that $z_{wk} = 0$ represents first twitch, and $z_{wk} = 1$ hatching. 

\begin{equation}
\label{eqn:time_norm}
    z_{wk} = \frac{t_{wk} - s_w}{h_w - s_w}
\end{equation}

Features used in each of the three models: \textit{Sides}, \textit{Pairs}, and \textit{Posture} are described with accompanying plots of their distributions throughout both development and within the embryo. All plots will feature the normalized time of observation $z_{wk}$ on the horizontal axis. The vertical axis unit will vary between angular measurements, distances (in $\mu m$), and ratios of each which have no unit. Plots with multiple subplots in will measure features calculated at using segments of the embryo posterior to anterior. The left-most plot will depict calculations using the tail pair while the right-most will depict the feature ending with the \textit{H0} pair at the anterior of the embryo. 

\subsubsection*{\textit{Sides}} \label{SI:QAP}

The graphical model uses scaled distances between pairs of nuclei. The first feature to analyze is the distance between paired cell nuclei (S1 Fig~\ref{fig:qap}-A): 

\begin{equation} \label{eqn:PD}
    \mathbf{PD}_i = \|\mathbf{L}_i - \mathbf{R}_i \|_2
\end{equation}

The distance between the left and right seam cell nuclei within a pair can be interpreted as the width of the embryo measured at the sampled seam cell locations. The first (left-most) subplot of S1 Fig~\ref{fig:qap}-A illustrates distances in microns between nuclei centroids of the tail pair for each observation. The tail pair distance is used for the initial pair selection rule $H_1$ across all models. The second set of distances form along the left and right sides of the worm (S1 Fig~\ref{fig:qap}-B). The lengths of chords between successive nuclei on each side are calculated: $\|\mathbf{L}_{i+1} - \mathbf{L}_i \|_2$ and $\|\mathbf{R}_{i+1} - \mathbf{R}_i \|_2$. Similar to the pair distances, side length observations are highly variant. 

\begin{figure}[!ht]
\centering
\includegraphics[width=\textwidth]{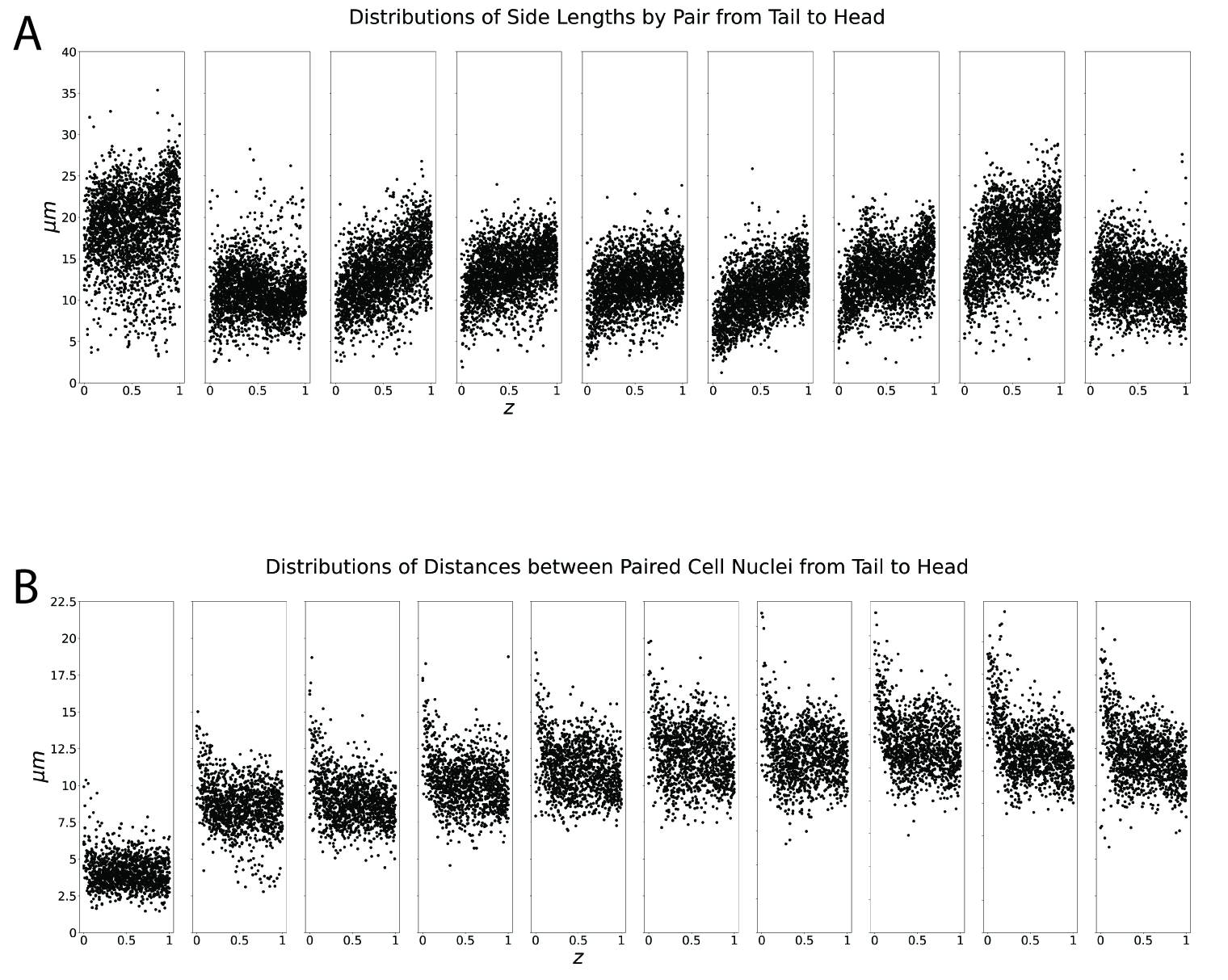}
\caption{\textbf{\textit{Sides} model features.} A) Distances between nuclei of lateral pairs. Notably, the tail pair distance (left-most panel) is constant throughout imaging. The tail pair distance informs the initial pair selection rule $H_1$. B) Chord lengths along left and right sides of the posture. Both quadratic features show high variance.}
\label{fig:qap}
\end{figure}

\subsubsection*{\textit{Pairs}} \label{SI:Pairs}

The \textit{Pairs} model uses hyperedges connecting two or three pairs of nuclei (four or six nuclei). The first four features measure pair-to-pair variation, while the latter two features use triplets of pairs to measure angles formed by the midpoints of the three pairs. The first pair-to-pair feature extends upon the use of pair distances to better describe the coiled worm. The ratio of sequential pair distances models the variation in width throughout the assigned nuclei (Fig~\ref{fig:pairs}-A): 

\begin{equation} \label{eqn:PDR}
    PDR_i = \frac{PD_i}{PD_{i+1}}
\end{equation}

Each feature's estimated mean is slightly greater than $1$, indicating that, on average, the worm is widening from tail to head. Another easily interpreted distance is the length of the chords connecting sequential pair midpoints. This is a more robust measure of worm length as side lengths vary more based upon the worm's folding (S1 Fig~\ref{fig:pairs}-B):

\begin{equation} \label{eqn:MD}
    MD_i = \| \mathbf{M}_{i+1} - \mathbf{M}_{i} \|_2
\end{equation}

The length of the chords connecting sequential pair midpoints is a more robust measure of worm length as side lengths vary more based upon the worm's folding. The cosine similarity is used to assess the degree to which sequential sides are pointing in the same direction (S1 Fig~\ref{fig:pairs}-C): 

\begin{equation} \label{eqn:CS}
    \phi_i =  \frac{(\mathbf{R}_{i+1} - \mathbf{R}_{i}) \cdot (\mathbf{L}_{i+1} - \mathbf{L}_{i})}{\| \mathbf{R}_{i+1} - \mathbf{R}_{i} \|_2 \| \mathbf{L}_{i+1} - \mathbf{L}_{i} \|_2} \in [-1, 1]
\end{equation} 

The final two pair-to-pair \textit{Pairs} features attempt to model two different types of \textit{twist} in the posture. The lateral and axial twists measures angles of rotation from lateral and posterior views, respectively (S1 Fig~\ref{fig:pairs}-D). 

\begin{align} \label{eqn:lat_ax}
    \mathbf{b}_1 &= \frac{\mathbf{L}_{i+1} - \mathbf{L}_i}{\| \mathbf{L}_{i+1} - \mathbf{L}_i \|_2} \\
    \mathbf{b}_2 &= \frac{\mathbf{L}_{i} - \mathbf{R}_i}{\| \mathbf{L}_{i} - \mathbf{R}_i \|_2} \\ 
    \mathbf{b}_3 &= \frac{\mathbf{R}_{i} - \mathbf{R}_{i+1}}{\| \mathbf{R}_{i} - \mathbf{R}_{i+1} \|_2} \\
    \mathbf{b}_4 &= \frac{\mathbf{R}_{i+1} - \mathbf{L}_{i+1}}{\| \mathbf{R}_{i+1} - \mathbf{L}_{i+1} \|_2} \\
    \mathbf{n}_1 &= \mathbf{b}_1 \times \mathbf{b}_2 \\
    \mathbf{n}_2 &= \mathbf{b}_2 \times \mathbf{b}_3 \\
    \mathbf{n}_3 &= \mathbf{b}_3 \times \mathbf{b}_4 \\
    \mathbf{c}_1 &= < \mathbf{n}_1 \times \mathbf{n}_2, \mathbf{b}_2 > \\
    \mathbf{c}_2 &= < \mathbf{n}_1, \mathbf{n}_2 > \\
    \psi_i &= \frac{1}{\pi} atan2(< \mathbf{n}_1 \times \mathbf{n}_2, \mathbf{b}_2>, < \mathbf{n}_1, \mathbf{n}_2>)) \\
\end{align} 

Axial twists present between a sequence of two pairs calculates the angle obtained by projecting the chord linking pairs onto each other (S1 Fig~\ref{fig:pairs}-E):

\begin{align} \label{eqn:ax}
    \mathbf{b}_1 &= \frac{\mathbf{L}_{i} - \mathbf{L}_{i+1}}{\| \mathbf{L}_{i} - \mathbf{L}_{i+1} \|_2} \\
    \mathbf{b}_2 &= \frac{\mathbf{R}_{i} - \mathbf{L}_i}{\| \mathbf{R}_{i} - \mathbf{L}_i \|_2} \\ 
    \mathbf{b}_3 &= \frac{\mathbf{R}_{i+1} - \mathbf{R}_{i}}{\| \mathbf{R}_{i+1} - \mathbf{R}_{i} \|_2} \\
    \mathbf{b}_4 &= \frac{\mathbf{L}_{i+1} - \mathbf{R}_{i+1}}{\| \mathbf{L}_{i+1} - \mathbf{R}_{i+1} \|_2} \\
    \mathbf{n}_1 &= \mathbf{b}_1 \times \mathbf{b}_2 \\
    \mathbf{n}_2 &= \mathbf{b}_2 \times \mathbf{b}_3 \\
    \mathbf{n}_3 &= \mathbf{b}_3 \times \mathbf{b}_4 \\
    \mathbf{c}_1 &= < \mathbf{n}_2 \times \mathbf{n}_3, \mathbf{b}_3 > \\
    \mathbf{c}_2 &= < \mathbf{n}_2, \mathbf{n}_3 > \\
    \tau_i &= \frac{1}{\pi} atan2(< \mathbf{n}_2 \times \mathbf{n}_3, \mathbf{b}_3>, < \mathbf{n}_2, \mathbf{n}_3>)) \\
\end{align} 

Angles along sides of the worm formed by triples of sequential nuclei approximate bend in the worm along each side. These bend angles are highly variant, especially frame-to-frame, in the same manner as side lengths in \textit{Sides} (S1 Fig~\ref{fig:qap}-B). Angles formed by pair midpoints exacerbate the computational burden as six nuclei are required, compare to three in a typical angle calculation, but the midpoint based angles are less variant than angles of each side (S1 Fig~\ref{fig:pairs}-F):

\begin{equation} \label{eqn:bend}
    \Theta_i = \frac{180}{\pi} \arccos{\frac{< \mathbf{M}_{i+1} - \mathbf{M}_{i}, \mathbf{M}_{i+2} - \mathbf{M}_{i+1} >}{\|  \mathbf{M}_{i+1} - \mathbf{M}_{i} \|_2 \| \mathbf{M}_{i+2} - \mathbf{M}_{i+1} \|_2}}
\end{equation}

Each angle $\Theta \in [0, 180]$ where $0$ would denote the worm perfectly folded upon itself, and $180$ would define a flat worm.  A second set of angles aims to approximate the posterior to anterior bend in the worm. The angles $\zeta_i$ are defines as the angles formed by fitted planes intersecting between pair midpoints (S1 Fig~\ref{fig:pairs}-G):

\begin{equation} \label{eqn:planar}
    \zeta_i = \frac{180}{\pi} \frac{< (\mathbf{R}_{i+1} - \mathbf{L}_{i+1}) \times (\mathbf{M}_{i+1} - \mathbf{M}_i), ((\mathbf{R}_{i+1} - \mathbf{L}_{i+1}) \times (\mathbf{M}_{i+1} - \mathbf{M}_{i+1})) >}{\| (\mathbf{R}_{i+1} - \mathbf{L}_{i+1}) \times (\mathbf{M}_{i+1} - \mathbf{M}_i) \|_2 \| (\mathbf{R}_{i+1} - \mathbf{L}_{i+1}) \times (\mathbf{M}_{i+1} - \mathbf{M}_{i+1}) \|_2}
\end{equation}

\begin{figure}[!ht]
\centering
\includegraphics[width=\textwidth]{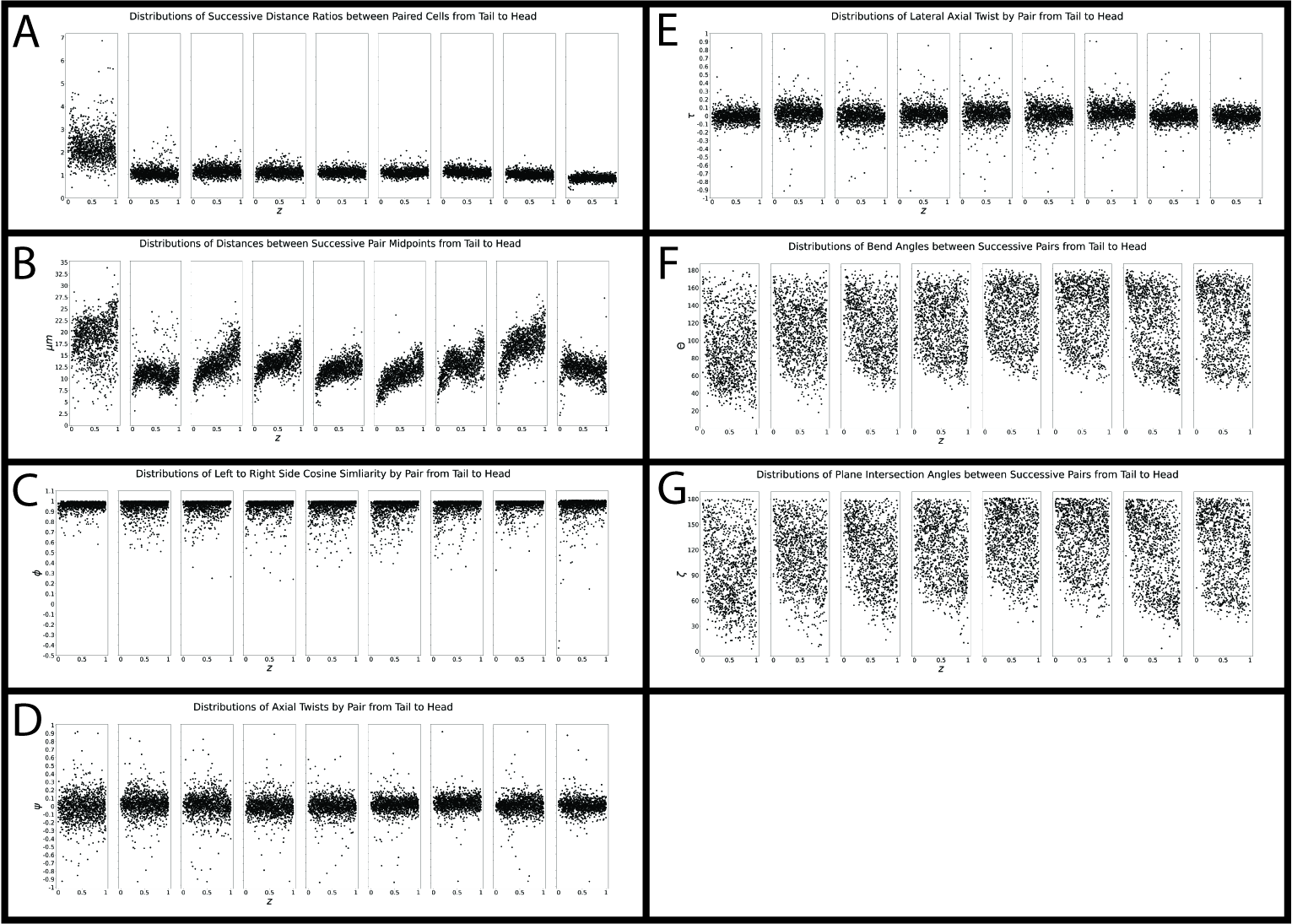}
\caption{\textbf{\textit{Pairs} model features.} A) Ratios of pair distances (Eq~\ref{eqn:PDR}). B) Distance between successive pair midpoints (Eq~\ref{eqn:MD}). C) Cosine similarities between successive left and right sides (Eq~\ref{eqn:CS}). D) Lateral axial twist angles (Eq~\ref{eqn:lat_ax}). E) Axial twist angles (Eq~\ref{eqn:ax}). F) Midpoint bend angles (Eq~\ref{eqn:bend}). G) Planar intersection angles (Eq~\ref{eqn:planar}).}
\label{fig:pairs}
\end{figure}

\subsubsection*{\textit{Posture}} \label{SI:Posture}

The \textit{Posture} model is comprised of all \textit{Pairs} features as well as the features defined by the summations of each local feature measurement throughout the hypothesized posture. Full posture features give insight into the changes in the embryo's shape throughout late-stage embryogenesis (S1 Fig~\ref{fig:posture}). Worm length follows an approximately logarithmic pattern.Total curvature follows a negative exponential pattern. Earlier on the worm is fatter and cannot bend as much. The worm elongates during development, allowing for sharper bends.

\begin{figure}[!ht]
\centering
\includegraphics[width=\textwidth]{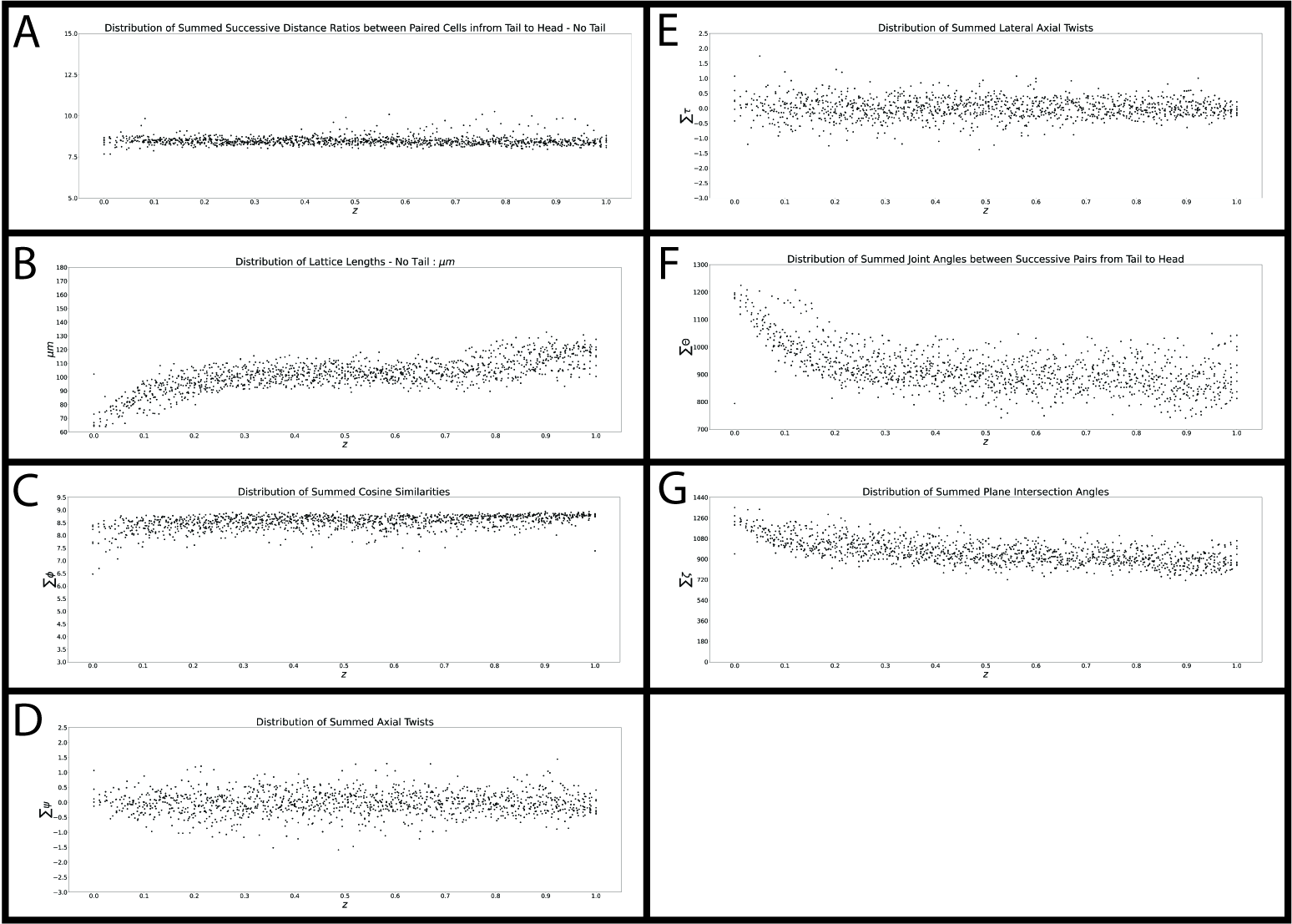}
\caption{\textbf{\textit{Posture} model features include all \textit{Pairs} features and posture-wide versions of \textit{Pairs} features.} A) Summed ratios of pair distances (Eq~\ref{eqn:PDR}). B) Summed distances between successive pair midpoints (Eq~\ref{eqn:MD}). C) Summed cosine similarities between successive left and right sides (Eq~\ref{eqn:CS}). D) Summed lateral axial twist angles (Eq~\ref{eqn:lat_ax}). E) Summed axial twist angles (Eq~\ref{eqn:ax}). F) Summed midpoint bend angles (Eq~\ref{eqn:bend}). G) Summed planar intersection angles (Eq~\ref{eqn:planar}).}
\label{fig:posture}
\end{figure}
\end{document}